\documentclass[letterpaper,11pt]{article}
\usepackage[utf8]{inputenc} % allow utf-8 input

\pdfoutput=1

\usepackage[T1]{fontenc}    % use 8-bit T1 fonts
\usepackage{hyperref}       % hyperlinks
\usepackage{url}            % simple URL typesetting
\usepackage{booktabs}       % professional-quality tables
\usepackage{amsfonts}       % blackboard math symbols
\usepackage{nicefrac}       % compact symbols for 1/2, etc.
\usepackage{microtype}      % microtypography
\usepackage{xcolor}         % colors

\usepackage{enumitem}
\usepackage{multirow}

\usepackage{amsmath}
\usepackage{amssymb}
\usepackage{amsthm}
\usepackage{subfig}
\usepackage{dsfont}
\usepackage[margin=1in]{geometry}
\usepackage{xspace}
\usepackage{tikz}

\usepackage{natbib}

\usepackage{thmtools,thm-restate}
\newtheorem{theorem}{Theorem}

\newtheorem{lemma}[theorem]{Lemma}
\newtheorem{proposition}{Proposition}
\newtheorem{definition}{Definition}

\newtheorem{remark}{Remark}

\newcommand{\fParantheses}[1]{\left ({#1} \right )}

\newcommand{\fBrackets}[1]{\left [{#1} \right ]}

\newcommand{\E}{\mathbb{E}}

\newcommand{\calD}{\mathcal{D}}

\newcommand{\Hcal}{{\mathcal H}}

\renewcommand{\hbar}{{\bar{h}}}

\newcommand{\EE}{\mathbb{E}}

\usepackage[suppress]{color-edits}
\addauthor{nh}{red}
\addauthor{sb}{magenta}
\addauthor{na}{blue}

\definecolor{mydarkblue}{rgb}{0,0.08,0.45}
\hypersetup{ %
  colorlinks=true,
  linkcolor=mydarkblue,
  citecolor=mydarkblue,
  filecolor=mydarkblue,
  urlcolor=mydarkblue}
\setcitestyle{authoryear,round,citesep={;},aysep={,},yysep={;}}

\title{Delegating Data Collection in Decentralized Machine Learning}

\makeatletter
\newcommand{\printfnsymbol}[1]{%
  \textsuperscript{\@fnsymbol{#1}}%
}
\makeatother

\usepackage{authblk}
\stepcounter{footnote}
\addtocounter{footnote}{-1}
\author{Nivasini Ananthakrishnan}
\author{Stephen Bates}
\author{Michael I. Jordan}
\author{Nika Haghtalab}
\affil{University of California, Berkeley}
\date{}                     %% if you don't need date to appear

\begin{document}

\maketitle
\begin{abstract}
  Motivated by the emergence of decentralized machine learning (ML) ecosystems, we study the delegation of data collection.  Taking the field of contract theory as our starting point, we design optimal and near-optimal contracts that deal with two fundamental information asymmetries that arise in decentralized ML: uncertainty in the assessment of model quality and uncertainty regarding the optimal performance of any model. We show that a principal can cope with such asymmetry via simple linear contracts that achieve $1-1/e$ fraction of the optimal utility. 
To address the lack of a priori knowledge regarding the optimal performance, we give a convex program that can adaptively and efficiently compute the optimal contract. We also study linear contracts and derive the optimal utility in the more complex setting of multiple interactions.

\end{abstract}

\section{Introduction}
\label{sec:intro}
The design of machine learning pipelines is increasingly a cooperative, distributed endeavor, in which the expertise needed for the design of various components of an overall pipeline is spread across many stakeholders.  Such
expertise pertains in part to classical design choices such as how much and what kind of data to use for training, how much test data to use for verification,
how to train a model, and how to tune hyper-parameters, but, more broadly, expertise may reflect experience, access to certain resources, or knowledge of local conditions. To the extent that there is a central designer, their role may in large part be that of setting requirements, developing coordination mechanisms, and providing incentives.

Overall, we are seeing a flourishing new industry at the intersection of
ML and operations which makes use of specialization and decentralization to
achieve high performance and operational efficiency.  Such an ML ecosystem creates
a need for new design tools and insights that are not focused merely on how the
designer could perform a task in this pipeline, but rather how she should delegate
it to agents who are willing and capable of performing the task on her behalf.
\emph{How should the designer interact with this ecosystem? How should she evaluate and compensate other agents for their work? How does the outcome of the delegated pipeline compare with the outcome if the designer were to perform the task by herself?}
In this work, we initiate the study of delegating machine learning pipelines through the lens of \emph{contract theory} and take a step towards answering these questions.

Contract theory provides a principal-agent perspective, where the principal---who is the designer interested in the outcome of the learning pipeline---can create a contractual arrangement---a menu of services and compensations---with an agent. At the heart of the issue is creating contracts that incentivize the agents, who may be more knowledgeable and skilled than the principal, to take the appropriate actions. The uncertain and data-centric nature of machine learning tasks brings to light interesting sources of knowledge asymmetry between the principal and the agent and requires extensions of classical contract theory.

Consider a scenario where a firm delegates a predictive task to an ML service provider. In this context, the service provider may offer the firm either a dataset for learning or a pre-trained predictive model based on that dataset. To ensure aligned incentives, the firm needs to assess the dataset or predictive model and design the payment structure for the service provider accordingly. Since the accuracy of the model is crucial to the firm as it directly influences revenue, a natural evaluation approach involves directly measuring the accuracy of the model that the service provider produces for the firm. Several challenges arise during this evaluation process. Firstly, the firm generally only has limited data in the form of historical data or data acquired shortly after deploying the model for evaluation. So there is inherent noise in the evaluation. Secondly, the firm lacks knowledge about the baseline accuracy that is realistically achievable. This makes it harder for the firm to reward the service provider in a way that yields accuracy close to the optimal accuracy. These challenges are due to two sources of uncertainty and asymmetry that we study in this work:

\begin{itemize}[leftmargin=*, itemsep=2pt, topsep = 2pt]
\item \emph{Hidden actions} (aka Moral Hazard): Contracts must compensate the agent for his effort towards creating a good outcome for the principal. These contracts therefore must depend on observable and verifiable outcome quality, such as the true accuracy of a classifier.
This is particularly challenging in machine learning pipelines, where the accuracy of the learned model is unknown a priori and random.
The principal may be able to invest in resources, such as large test sets, that reduce this uncertainty at a cost and better evaluate the agent's effort. 
An important consideration here is determining 
whether the principal must accurately verify the outcome or instead incentivize the agent in the first place to ensure a high-quality outcome.

\item \emph{Hidden state} (aka Adverse Selection): Effective contracts use the knowledge of the best achievable outcome to incentivize the agents to work towards such outcomes. 
This is challenging in machine learning pipelines where the true error of the best model is unknown. Furthermore, generic methods that estimate the optimal accuracy
tend to use almost as many resources as are needed to learn a model of that accuracy. Here again, we must ask whether contracts exist that appropriately incentivize the agent to perform his best while knowing very little about the optimal possible accuracy.
\end{itemize}

\subsection{Our results}
We consider performance-based contracts where the agent is compensated as a function of the estimated accuracy of the learned model. The principal's utility is the true accuracy of the learned classifier minus the monetary transfer she makes to the agent. We model the agent in two delegation settings. In each we contrast the principal's utility through contracting with and without information asymmetry. Borrowing terminology from the contract theory literature, we refer to the hypothetical scenario without information asymmetry as the \emph{first-best scenario} and the resulting optimal utility as the \emph{first-best utility}, which we use as a benchmark.

\paragraph{Single-round of interaction.} We address both types of information asymmetry---hidden state and hidden action---creating contracts specific to each while also evaluating their efficacy when both asymmetries coexist. For hidden actions, our linear contract based on a single test point (Proposition~\ref{thm:LinContractApprox}) ensures at least $1 - 1/e$ fraction of the first-best utility. This guarantee continues to hold even with hidden state if the agent's sampling cost is low (Theorem~\ref{cor:advSelLinear}). 
    
    For the hidden state challenge with $n$ possible states, we derive an optimal contract by solving a convex optimization problem with $O(n^2)$ constraints. Section~\ref{sec:mediumtest} describes how this contract's optimality improves as the principal's test set size increases.

\paragraph{Multiple rounds of interaction.}
In Section~\ref{sec:multiround}, we analyze a multi-round delegation setting where the agent is uncertain about the delegated task and uses feedback over rounds to learn the principal's requirements and collect relevant samples. We define the principal's regret and establish a tight $\Theta(T^{3/4})$ bound on this regret through repeated linear contracts over $T$ rounds. {
This shows that linear contracts are also powerful approximations of optimal utility in multi-round settings.
In comparison, we obtain a strictly better regret of $O(T^{2/3})$ for multi-round first-best contracts.
}

\subsection{Related work}\label{sec:relatedWork}
    There is a rich literature on contract theory in economics \citep[see, e.g.,][]{laffont2009theory,bolton2004contract}.
    More recently, there has been work on algorithmic and statistical aspects of contract theory~\citep{carroll2015robustness, dutting2019simple, dutting2020complexity, bates2022principalagent, alon2022bayesian} which include results on approximation by simple contracts. These results hold for either finite actions or outcomes, and thus are not directly applicable to our setting, which involves infinite actions and a continuous space of outcomes. Working in such spaces requires utilizing the structure of our problem, and specifically exploiting fundamental results on statistical minimax rates.

  The \emph{pricing} of data has been considered for various purposes and considerations~\citep{bergemann2019markets,acemoglu2022too,cai2015optimum, ho2016adaptive} including in learning problems~\citep{agarwal2019marketplace,chen2022selling}. The latter study the pricing of previously collected data to incentivize the seller and buyer to be forthright about the valuation and quality of their data, respectively. We are interested instead in pricing for the purpose of incentivizing the data collecting agent to exert effort to collect data. Some of these papers also consider incentivizing high-quality data labelling by relying on multiple labellers who can be compared. We study delegation of learning in the setting of a single agent.

An adversarial perspective on the delegation problem has been considered for machine learning from the lens of interactive proofs. In this line of work~\citep{goldwasser2021interactive,chiesa2018proofs}, the principal wants to fully verify the effort of an agent who may be an adversary that is interested in getting his effort verified. While they deal with similar challenges, such as not knowing the optimal achievable error, they do not consider incentivizing the agent (via contracts and compensations) to improve the outcome.

Concurrent work by~\cite{saig2023delegated} studies a similar setting of incentivizing data collection for classification. They characterize the optimal contract for a given test set size, under the hidden action challenge, as a threshold contract when the agent has two choices for actions. They provide conditions which make the threshold contract optimal even for additional actions. They study empirically the effect of the hidden state challenge. We provide results for an arbitrary number of actions and propose a contract that is based on a single test sample that is optimal relative to what is achievable without the hidden action and hidden state challenges. We show that this contract is robust to hidden state challenges in many cases and describe other approaches of dealing with hidden state challenges outside of these cases.

\section{Model}\label{sec:dataCollectionModel}

We have a task distribution $\mathcal{D}$ representing the joint distribution over the domain and label set. The principal aims to learn a classifier $h$ that achieves high accuracy on $\mathcal{D}$, denoted by $1 - L_{\mathcal{D}}(h)$. To accomplish this, the principal delegates the task to an agent who selects the number of samples to collect and trains a classifier. We prioritize the collection of samples as the primary effort, considering it more significant than classifier training. The principal's primary objective is to incentivize high-quality data collection, leading to the development of an accurate classifier. To evaluate the performance of the model obtained through delegation, the principal possesses an independent test set consisting of independently and identically distributed (i.i.d.) points drawn from the distribution $\mathcal{D}$. The principal utilizes this test set to evaluate the learned classifier's accuracy.

As in many other delegation settings, the principal faces the hidden state and hidden action challenges when delegating learning. While the principal desires to construct a contract based on the true accuracy of the learned model, {$1 - L_{\mathcal{D}}(h)$}, they can only obtain a noisy estimate of this value using test data. Our focus is on scenarios where the size of the test dataset is not excessively large. If the test dataset is too large, it becomes more beneficial for the principal to learn a model using their own test data rather than delegating the data collection process. Even when the estimate of the learned model's accuracy has negligible noise, the principal still faces the hidden state challenge, i.e., the principal does not know how to value the accuracy since she does not know the optimal error achievable. {We use $1-\theta$ to indicate the optimal accuracy achievable on $\mathcal{D}$.}
Assigning a low payment for the model's accuracy when the optimal error, {$\theta$,} is high would result in negative agent utility, discouraging agent participation. Conversely, assigning a high payment when the optimal error is low might incentivize the agent to collect a smaller dataset than is optimal for the principal.

The delegation process begins with the principal publishing a contract which is a mapping from test accuracy to payment for the agent. Seeing the contract, the agent collects data and provides a classifier to the principal. The principal then executes the contract by evaluating the classifier on her own test set. The principal pays the agent the amount specified by the contract for the measured test accuracy. We assume that the principal can commit to a test set in advance and that this test set is not accessible to the agent until the contract is executed after the agent's data collection.

\textbf{Utilities.} Upon receiving a classifier with accuracy $a$ for the task distribution $\mathcal{D}$ and paying the agent $t$, the principal gets utility $a - \beta t$ for some constant $\beta > 0$. The agent exerts effort $\alpha$ per sample it collects. So the utility for the agent receiving payment $t$ by collecting $n$ samples is $t - \alpha n$.

\textbf{Outcome as a function of the agent's action.} We assume that when the agent collects $n$ samples,\footnote{We will consider agent's action as continuous and the true sample size is a rounding of the action.} the  classifier's observed accuracy on the principal's test set drawn from $\mathcal{D}$ is drawn from a distribution with mean $1 - \theta - \frac{d}{n^p}$ {and variance that is determined by the size of the test set.}
 The constant $d$ depends on the complexity of the training algorithm and the constant $p$ describes the rate of decay of the excess error. These rates are motivated in part by minimax statistical rates {and scaling laws}.

Even though minimax rates are typically upper bounds, we treat them as exact rates in the main body and defer the discussion on the implications of treating them as upper bounds to Appendix~\ref{sec:errorBounds}.

\begin{remark}[VC dimension bound]
    An algorithm that PAC-learns a function class $\Hcal$ with VC dimension $d$ using $n$ i.i.d.\ samples drawn from $\calD$ and returns a classifier $h$ satisfying $L_\calD(h) \leq \theta(\calD,\Hcal) + C\sqrt{d/n}$, where $\theta(\calD,\Hcal) = \min_{h \in \Hcal} L_\calD(h)$. This is minimax-optimal as there is a distribution $\calD$ such that $L_\calD(h) \geq \theta(\calD,\Hcal) + C\sqrt{d/n}$.
\end{remark}

\begin{remark}[Linear regression model]\label{rem:linregression}
    In a $d$-dimensional linear model with covariates $x_i \sim \mathcal{N}(0,\Sigma)$ and outcomes $y_i = \beta^t x_i + \epsilon_i$, where $\epsilon_i \sim \mathcal{N}(0,\sigma^2)$ for $i \in [n]$, the Ordinary Least Squares (OLS) estimator $\hat{\beta}$ satisfies the property $\EE[(x^t\hat{\beta} - y_i)^2] = \sigma^2 \fParantheses{1 + O(d/n)}$.
\end{remark}

\textbf{First-best contracts.}
As a benchmark for the best performance we can hope to achieve, we first consider the problem in an idealized setting without the hidden state and hidden action challenges. This is when the principal knows the optimal error $\theta$ and the mapping between the agent's action $n$ and the test accuracy of the resulting model is deterministic (i.e., is exactly $1 - \theta - \frac{d}{n^p}$). The optimal contract in this idealized setting is called the \emph{first-best contract}. The next proposition provides a closed form for this contract.

\begin{proposition}[First-best contract]\label{thm:firstBest}
    For any set of problem parameters $\theta \in [0,1), d, p, \alpha, \beta > 0$, the first-best contract offers payment $\alpha n^*$ when the test accuracy is at least $1 - \theta - d/{n^{*p}}$, where $n^* = (pd/\alpha \beta)^{1/(p+1)}$.
\end{proposition}
One way to interpret the first-best contract is that it asks the agent to  collect $n^*$ samples and compensates the agent exactly for $n^*$ samples. 
Without hidden state or hidden action, the first-best contract yields zero utility to the agent. In this idealized scenario, the principal's utility due to the first-best contract is called the \emph{first-best utility} and serves as a benchmark for comparison in our analysis of delegation. While first-best utility is used as a benchmark,  the first-best contract itself may not be optimal due to existing randomness in test accuracy (hidden action).
Additionally, each optimal error value $\theta$ leads to a different first-best contract, which is not implementable when the principal doesn't know the $\theta$ parameter exactly (hidden state). 
When dealing only with hidden action but known $\theta$, the principal's goal is to set up a contract specified for $\theta$ that deals with the randomness in the test accuracy to recover some fraction of the first-best utility. However, when both actions and states are unknown (uncertainty in $\theta$ and test accuracy) the contract must ensure good principal utility for a range of possible states $\theta$.

\textbf{Linear contracts.}
As opposed to first-best contracts that can be quite complex, \emph{linear contracts} are simple contracts that compensate an agent by a linear function of the test accuracy. That is, a $c$-linear contract for parameter $c \in \mathbb{R}^+$ assigns payment $T_c(a) = c \times a$ when the test accuracy is $a$.

Linear contracts must have non-negative parameter $c$, since the principal cannot make negative payments to the agent.

\section{Optimality of Linear Contracts}
\label{sec:dataCollectionResults}
In this section, we aim to find near-optimal contracts in the realistic scenario with hidden state and hidden actions, recognizing that the first-best contract may not be optimal. Our main  result is that 
{a linear contract compensating the agent based on the test (and not true) accuracy is approximately optimal across all possible contracts for the principal.}
Moreover, the slope of the linear contract has an explicit value that is the same across a wide range of $\theta$ making it possible to deal with both hidden state and hidden state challenges.

A crucial advantage of our linear contract is that it works with any unbiased estimator of the accuracy of the learned model. Therefore, even a test set of size one suffices to enact this contract. We state our main results in this section and defer their formal proofs to Appendix~\ref{app:proofs}.

Consider the hidden action (but known state) challenge where the principal knows $\theta$ but not the random mapping from the agent's action $n$ to test accuracy. This mapping has a mean $1-\theta-\frac{d}{n^p}$) and a variance dependent on the test set size. The variance is finite but possibly arbitrarily large. This setting includes the delegation problem where the principal has as little as just a single sample $x\sim \calD$ in her test set. Furthermore, we assume that beyond knowing the mean 
of the distribution of the test accuracy $1-\theta-\frac{d}{n^p}$, the distribution can be arbitrary and unknown to the principal.

Our main result is that we can design an approximately optimal linear contract.
Furthermore, under a wide range of problem parameters $\theta$, the principal does not even need to know the optimal error to construct this contract. This allows us to deal with both hidden state and hidden action challenges.
Our results in fact show a stronger comparison, that  linear contracts approximate not just the optimal utility but also the first-best utility. This is quite a strong guarantee as there is often no contract that can achieve the first-best utility in presence of the hidden action challenge.

Before we state our main theorem, we start with the following proposition which deals only with the hidden action challenge while assuming that optimal error $\theta$ is known to the principal.
Our main result in Theorem~\ref{cor:advSelLinear} follows from this proposition and shows that the linear contract in this proposition is also a good choice in more general settings.

\begin{proposition}[Linear contracts are approximately optimal when optimal error is known]\label{thm:LinContractApprox}
For any set of problem parameters $\theta \in [0,1), d, p, \alpha, \beta > 0$, if the principal knows $\theta$ {(but not the distribution of the test error)} she can construct a linear contract that brings an expected utility that is at least $1 - 1/e$ times the first-best utility. {Furthermore this contract only requires a single test sample.}

The linear contract $c^*$ that achieves this approximately optimal utility is the following:
\[c^* = \max \fParantheses{\frac 1 {\beta(p+1)^{\frac{p+1}{p}}}, \frac{\alpha d^{\frac 1 p}}{p} \cdot \fParantheses{\frac{p+1}{1-\theta}}^{\frac{p+1}{p}}}.\]
At a finer level, this linear contract approximates the first-best utility by a factor of 
\[1 - \frac{1}{(p+1)^{\frac{p+1}{p}}} {\geq 1- \frac 1e}.\]
\end{proposition}

{Let us first note that previous work~\citep{alon2022bayesian,dutting2019simple} has provided constant approximation guarantees 
but is limited to settings where the agent's action set is a finite set or where the ratio of the maximum and minimum reward for the principal is bounded by $H$.  In the former case, an approximation ratio of $1/2$ is obtained and in the latter the ratio is $1/2\log(H)$. Neither of these conditions hold in our settings, as the action is the number of samples collected and is unbounded and the reward can take any value in $(0,1)$. Instead, we use the structure of first-best contract (Proposition~\ref{thm:firstBest}}), the linearity of contracts, and the structure of the utility functions to obtain this $1-1/e$ approximation guarantees.

\begin{proof}[Proof sketch of Proposition~\ref{thm:LinContractApprox}]
    The full details are deferred to the appendices; here we provide some intuition and a proof sketch. 
    Underlying the proof is the linearity of expectation and the fact that the agent is expectation-maximizing. Under a linear contract $c$, the expectation-maximizing agent aims to maximize $\mathbb{E}[c \cdot a(n) - \alpha n] = c \cdot \mathbb{E}[a(n)] - \alpha n$, where $a(n)$ is the test-set accuracy of a model trained on $n$ samples drawn from an unknown distribution with mean $1 - \theta - d/n^p$. The only  distribution-dependent quantity in this maximizing objective is the expected accuracy $\mathbb{E}[a(n)]=1 - \theta - d/n^p$. So the agent's action and hence the principal's contract design only depends on the expectation of the test accuracy and not on the exact distribution of the test accuracy. Next we sketch a proof for the approximation result and use the structured way the expected accuracy depends on the number of samples drawn.

    Note that $c^*$ is the maximum of two terms.  Let us denote these terms by $c_1,c_2$.
    Given a linear contract with parameter $c$, the agent's best response is to choose $n$ so as to maximize $u(n;c) = c \fParantheses{1 - \theta - d/n^p} - \alpha n$. The maximizing value is $n(c) = \fParantheses{{cdp}/{\alpha}}^{\frac{1}{p+1}}$. By setting $c$ large enough, we have $u(n(c), c) \geq 0$ where $c_2$ is the threshold above which this holds. So the value of $c_2$ is set to ensure the agent gets non-negative utility from participating.

When $c_1 \geq c_2$, $c_1$ satisfies the participation constraint. By computing the principal's utility from the linear contract $c_1$ using the expression for the agent's best response, we see that it is $1-\beta c_1$ times the first-best utility. Moreover, we have $1 - \beta c_1 = 1 - 1/(p+1)^{\frac {p+1} p}$. It turns out the same upper bound holds for the approximation ratio of the linear contract $c_2$ to the first-best utility when $c_2 \geq c_1$. This upper bound is decreasing in $p$ and the limit as $p \rightarrow 0$ is $1-1/e$.
\end{proof}

Importantly, by inspecting the contract in Proposition~\ref{thm:LinContractApprox}, we see that in many cases it does not depend on problem-specific parameters like the optimum error. This makes $c^*$ deployable in practice.

The optimal-error-parameter-agnostic linear contract is appropriate when the cost per sample collection is small enough and when the optimal error is low enough. As a result, when $\alpha$ is small, we can relax the assumption that the principal knows the exact optimum error $\theta$ to that the principal knows that $\theta$ lies in a certain range. Moreover, even under this relaxation, linear contracts are still approximately optimal. This is stated as the following theorem.

\newcommand{\thetahigh}{\overline{\theta}}

\begin{theorem}[Main result]\label{cor:advSelLinear}
For any $d, p, \beta > 0$, consider the linear contract $\bar{c} = \frac 1 {\beta(p+1)^{\frac{p+1}{p}}}$. For any $\thetahigh \in [0,1)$, suppose the optimum error $\theta$ is any value in $[0, \thetahigh)$ and that $0 < \alpha \leq \frac{p}{\beta d^{1/p}} \fParantheses{\frac{1-\thetahigh}{(p+1)^2}}^{\frac{p+1}{p}}$. Then, $\bar{c}$ has utility at least $(1 - 1/e)$ times the first-best utility. 
\end{theorem}

Note that $\bar{c}$ is constructed based on $p$ (error decay rate) and $\beta$ (how the principal values accuracy relative to payment). The principal knows these quantities. In contrast, the optimal contract requires additional knowledge, such as $\theta$ (optimum error) and $\alpha$ (agent's cost per sample). The theorem demonstrates a simple contract that requires less knowledge but remains approximately optimal in utility.

\section{Extensions}\label{sec:mediumtest}
\textbf{Medium test set regime.}
In our analysis, we have examined the impact of the hidden action challenge when dealing with a small test set size. The significance of the hidden action challenge diminishes as the test set size increases, as the principal can obtain highly accurate estimates of the model's accuracy. However, when the test set becomes too large, delegation loses its value since the principal can independently learn an accurate model without delegation. Is there a regime in which the test set size is large enough for hidden action to not be significant while also being small enough for the principal to benefit from delegating data collection? In this section, we demonstrate the existence of such a regime, referred to as the ``medium test set regime.'' Later, we outline how we can capitalize on the larger size of the test set to achieve stronger results.

The sample complexity for learning an $\epsilon$-optimal model is $\Theta(d/\epsilon^2)$. In particular, this bound is linear in the training algorithm's complexity which can be problematic when using highly complex training algorithms. We say that the medium test set regime exists, if the sample complexity for hidden action is significantly smaller than 
$\Theta(d/\epsilon^2)$,  
where $\epsilon$ captures the significance level of hidden action which we will make precise in the following definition.

\begin{definition}[Insignificance of hidden action at level $\epsilon$]
    In a finite test set setting with hidden action, for any optimal error parameter $\theta$, let $\mathrm{OPT}$ denote the optimal expected utility of contracting. We say that hidden action is insignificant at level $\epsilon$, for any $\epsilon > 0$, if the expected utility of the first-best contract based on $\theta$ in this setting is at least $\mathrm{OPT} - \epsilon$.
\end{definition}
We next state a theorem giving the sample complexity of the principal's test set to achieve insignificance of the hidden action.
The sample complexity stated in the theorem is logarithmic in $d$ while learning would have required a number of samples linear in $d$. This demonstrates the existence of a medium test set regime where it is possible to employ delegation without considering hidden action.

\begin{theorem}[Sample complexity for insignificant hidden action]\label{thm:sampleCompNoMoralHazard}
    For any $\epsilon > 0$, if the principal has a test set of size $O\left (\frac 1 {\epsilon^2} \log \frac d {\epsilon} \right)$, then hidden action is insignificant at the level $\epsilon$. 
\end{theorem}

\subsection{Optimal contracts for hidden state}\label{sec:hiddenStateContract}
By ignoring hidden action in the medium test set regime, we can hope to design contracts with stronger guarantees. Previously we were able to design contracts with high utility when the optimal error lies in a particular range given in Theorem~\ref{cor:advSelLinear}. By ignoring hidden action, we can design contracts with utility guarantees for when the optimal error lies in any arbitrary set. When the principal holds a finitely supported prior belief over the optimal error value, we show how to compute the optimal contract by setting up a convex optimization problem. We also describe some qualitative properties of the optimal contract in this setting.  

When we ignore the hidden action challenge, we can assume that the observed accuracy is deterministic in the agent's action. That is, when the agent collects $n$ samples, the observed accuracy is $a(n,\theta) = 1 - \theta - d/{n^p}$. We assume that the principal holds a prior belief on the optimal error but does not know the exact value. The agent knows more about the optimal error since he collects data that informs him more about the optimal error. We assume that the agent knows the exact optimal error. We start by analyzing the optimal contract in this setting. Later in Section~\ref{sec:DataLearn}, we discuss how to design contracts in the more realistic setting of the agent learning the optimal error instead of knowing this value exactly. And we show that the utility guarantees by making the perfectly aware agent assumption still hold approximately in the more realistic case with a learning agent.

Let us analyze the optimization problem for computing the optimal contract. Let the finite support of the prior over optimal error be $\{\theta_1, \ldots, \theta_N\}$. The principal puts forth a contract of accuracy-payment pairs  $\{(a_i, t_i) : i \in [N]\}$ with the pair $i$ intended for when the optimal error is $\theta_i$.\footnote{This is implied by the revelation principle that states that, with hidden state, any delegation mechanism is equivalent to an \emph{incentive compatible} mechanism where all agents inform their private information to a planner who then recommends actions.} Let us denote the expected accuracy from collecting $n_i$ when optimal error is $\theta_i$ by $a_i = a(n_i, \theta_i)$. Here $n_i$ is the number of samples the agent would collect to achieve accuracy $a_i$ when optimal error is $\theta_i$. The principal optimizes over $(n_i,t_i)_{i \in [N]}$. The constraints of the optimization problem for the principal's contract design for hidden state are one of two types. The first type of constraint is the participation constraint, which ensures that the agent is adequately compensated for his effort when he chooses the contract intended for the optimal error. For each $i \in [N]$, the participation constraint $(\text{PC}_i)$ can be expressed as $\alpha n_i \leq t_i$, where $\alpha$ represents the compensation rate.

The second type of constraint is the incentive compatibility constraint to ensure that the agent chooses the option intended in the contract for the optimal error. For any $i,j \in [N]$, the corresponding incentive compatibility constraint is that when the optimal error is $\theta_i$, the utility of choosing $(a_j,t_j)$ is worse for the agent than choosing $(a_i,t_i)$. The number of samples the agent would choose to achieve $a_j$ accuracy under optimal error $\theta_i$ is $n_{ij}$ such that $a_j = a(n_{ij}, \theta_i)$.\footnote{Note that all accuracies cannot be achieved for all optimal errors.  If no such $n_{ij}$ exists, an incentive compatibility constraint is not needed.} The constraint $(\text{IC}_{ij})$ is $t_j - \alpha n_{ij} \leq t_i - \alpha n_i$. Due to the structure of $a(n,\theta)$, the IC constraints are convex. 
The principal's expected utility which it maximizes is $\sum_{i=1}^N \nu(\theta_i) (a_i - \beta t_i)$. So the contract design problem is the following optimization problem:
\vspace{-0.3in}
{  
  \setlength{\abovedisplayskip}{6pt}
  \setlength{\belowdisplayskip}{\abovedisplayskip}
  \setlength{\abovedisplayshortskip}{0pt}
  \setlength{\belowdisplayshortskip}{3pt}
  \begin{center}
\begin{equation}
\begin{aligned}
\min_{(n_i,t_i)_{i=1}^N} \quad & \sum_{i=1}^N \nu(\theta_i) (a_i - \beta t_i)\\
\textrm{s.t.} \quad & \alpha n_i \leq t_i, \quad i \in [N]\\
  &  t_j - \alpha n_{ij} \leq t_i - \alpha n_i, \quad i,j \in [N]  \\
  & n_i, t_i \geq 0, \quad i \in [N].
\end{aligned}
\tag{Opt}\label{opt:advSelection}
\end{equation}
\end{center}
}%

\textbf{Qualitative insights on the optimal contract.} We derive the following insights (see Figure~\ref{fig:optContractProperties}) when there are two values for the optimal error, $\theta_1 < \theta_2$, in the Appendix~\ref{app:closedFormOpAdvSel}. These properties also hold more generally for finitely supported beliefs and have been studied for classical contract design for many other delegation problems~\cite{laffont2009theory}.

\begin{itemize}[itemsep=0pt, topsep=0pt, leftmargin = *]
    \item \emph{Decreased utility}. The principal gets lower utility than the first-best utility and this utility decreases as $\Delta \theta = \theta_2 - \theta_1$ increases.
    
    \item \emph{Information rent.} In the first-best contract, the agent gets no more payment than to compensate his effort. That is, $t = \alpha n$. Under hidden state, for problems with lower optimal error, the agent gets positive utility. This information rent is to incentivize the agent to not pretend the problem is harder and exert lower effort to achieve an accuracy that requires more effort if the problem was harder. 

    \item \emph{Downward distortion.} The first-best contract calls for the agent to collect a particular number of samples regardless of the optimal error. Under hidden state, when the problem is harder, agents are asked to collect fewer samples compared to the first-best contract. When the problem is the easiest in the support, the agent is asked to collect the same number of samples as in the first contract.
\end{itemize}

\begin{figure}[!ht]
    \subfloat[{\small Information rent vs $\Delta \theta$}\label{subfig:rentVsGap}]{%
      \includegraphics[width=0.3\textwidth]{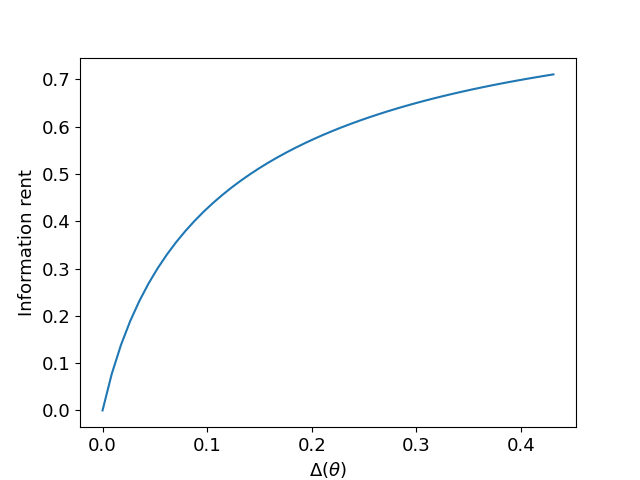}}
    \hfill
    \subfloat[{\small Downward distortion magnitude vs agent's $\Delta \theta$ }\label{subfig:optVsAlpha}]{%
      \includegraphics[width=0.3\textwidth]{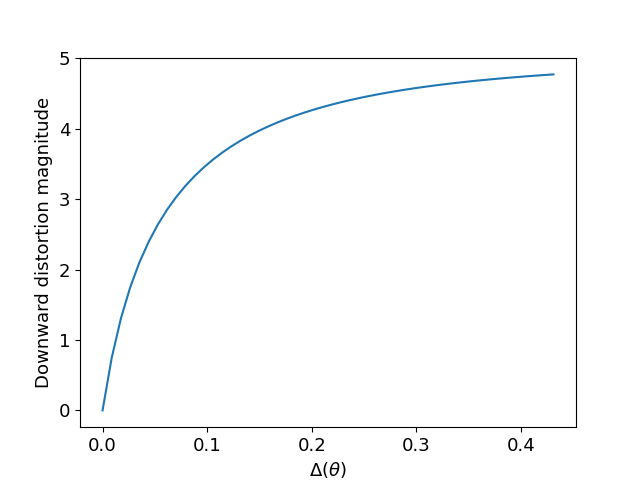}
    } 
    \hfill
     % \centering
    \subfloat[{\small Optimal utility vs $\Delta \theta$}\label{subfig:optVsGap}]{%
      \includegraphics[width=0.3\textwidth]{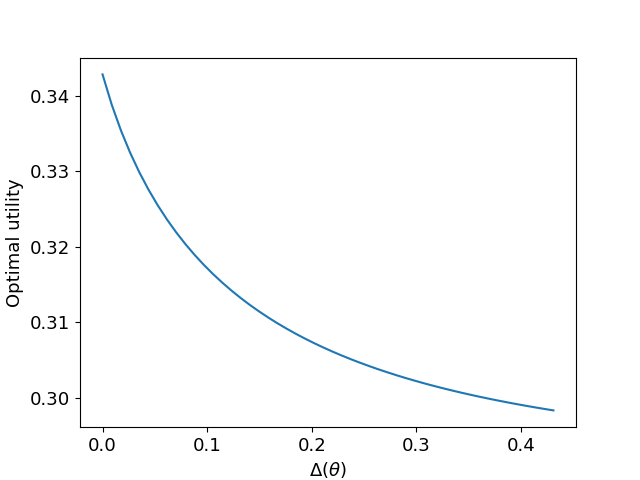}
    }
    \caption{Variation of utility, information rent and downward distortion magnitude with the gap in optimal error values ($\Delta \theta$). Information rent is the utility the agent makes under the lower optimal error problem. Downward distortion magnitude is how many fewer samples the agent collects compared to the first-best contract under the higher optimal error problem.}
    \label{fig:advSelOptVsProbParams}
  \end{figure}\label{fig:optContractProperties}

\subsection{Designing contracts against state-learning agents}
\label{sec:DataLearn}

In Section~\ref{sec:hiddenStateContract} and particularly \ref{opt:advSelection},  we assumed perfect knowledge of the hidden state ($\theta$) by the agent. However, in reality, the agent does not know the optimal error beforehand. Instead, as the agent executes the contract, he learns more about the optimal error and adapts his actions accordingly. To design a contract for such a state-learning agent, the principal would need to predict the agent's response to the contract. However, this is challenging for arbitrary contracts since the principal would require knowledge of the agent's exact learning strategy, which is often unreasonable. Therefore, we focus on analyzing simple contracts for which we can easily derive the agent's response. We demonstrate numerically that the utility achieved with these simple contracts is close to the utility we previously derived for state-aware agents, which we refer to as ``state-aware utility.'' 
This provides evidence that qualitative insights we derived about the state-aware utility in Section~\ref{sec:dataCollectionModel} and Section~\ref{sec:dataCollectionResults} are applicable in the more realistic case of a state-learning agent.
 We focus on the case where the optimal error can take one of two possible values, $\theta_1 < \theta_2$, but these design principles also extend to more possible values of the optimal error. The simple contract we consider, which we call the \emph{state-learning contract}, is the best of two simple contracts: optimal \emph{pooling} and \emph{separating} contracts. 

\textbf{Separating contract.} Separating contracts allow the agent to perfectly infer the hidden state while executing the contract. These are incentive-compatible contracts that ask agents to collect $n_1, n_2$ samples under optimal errors $\theta_1 < \theta_2$ respectively. Additionally, $n_1, n_2$ are such that the agent can successfully infer the optimal error after collecting $\min(n_1,n_2)$ samples. The agent's response to this contract would be to first collect $\min(n_1,n_2)$ samples and decide whether to collect more depending on the inferred optimal error. The agent's successful inference of the optimal error makes computing optimal separating contracts similar to the contract design problem against a state-aware agent, which was solving the optimization problem~\ref{opt:advSelection}. The new optimization problem that yields optimal separating contracts has the same objective and constraints as~\ref{opt:advSelection} with the added constraint that $n_1, n_2 > n_0$ for some $n_0$ that we will describe soon. The additional constraint ensures that the agent knows the optimal error (with high probability) after collecting $\min(n_1,n_2)$ samples.

To determine the value of $n_0$, we rely on assumptions about the agent's learning strategy. We assume that the agent can distinguish between $\theta_1,\theta_2$ with high probability using $k/(\Delta \theta)^2$ samples. Here $\Delta \theta = \theta_2 - \theta_1$ and $k$ is a constant reflecting the degree of assumption made about the agent's efficiency. A lower value of $k$ is a stronger assumption, assuming a more efficient agent. This assumption on the agent's learning strategy is more reasonable compared to assuming precise knowledge of the agent's learning strategy.

When $n_0$ is small enough that the added constraint $n_1,n_2 > n_0$ is not active, the state-learning agent is behaving exactly as the state-aware agent, so our results from Section~\ref{sec:hiddenStateContract} apply. On the other hand, if $n_0$ is large (which happens when $\Delta \theta$ is small) the additional constraint becomes too restrictive and the utility becomes low. In this case, another approach works well. 

\textbf{Pooling contract.} In pooling contracts, the agent has no incentive to learn the optimal error. The pooling contract asks the agent to achieve one accuracy level $\bar{a}$ regardless of the optimal error. The payment for this accuracy is set to ensure the agent can get nonnegative utility regardless of the optimal error. It is again straightforward to understand the agent's response to this contract. Suppose $\bar{a}$ can be achieved by collecting $\bar{n}_1 < \bar{n}_2$ samples under optimal errors $\theta_1 < \theta_2$ respectively. To execute this contract, the agent starts collecting $\bar{n}_1$ and sees if it achieves $\bar{a}$ accuracy. If it does not, he collects $\bar{n}_2-\bar{n}_1$ more samples since this action is guaranteed to yield nonnegative utility. Furthermore, collecting fewer or no additional samples results in less than $\bar{a}$ expected accuracy and hence zero payment even though the agent exerted effort. 

A pooling contract does not let agents differentiate actions for different optimal errors and would be sub-optimal for this reason. However, when the difference in both problems is not significant i.e., $\Delta \theta$ is low, the benefit to the principal for distinguishing the agents is low. In summary, the separating contract has good utility when $\Delta \theta$ is large and the pooling contract has good utility when $\Delta \theta$ is small. By deploying the contract of the two with the higher utility, we can hope to have good utility for all values of $\Delta \theta \in [0,0.5]$. 

\textbf{Numerical results.} We  compute the utility  difference between the state-aware contract and the state-learning contract, varying problem parameters $\Delta \theta = \theta_2 - \theta_1$ and $k$. We highlight a few observations (see Fig~\ref{fig:typeLearningContract}), that reflect the intuition we used to design the approach for state-learning contracts. 

\begin{figure}[!ht]
    \subfloat[{\small Utilities of different contracts vs $\Delta \theta$}\label{subfig:typeAwarePoolSep}]{%
      \includegraphics[width=0.3\textwidth]{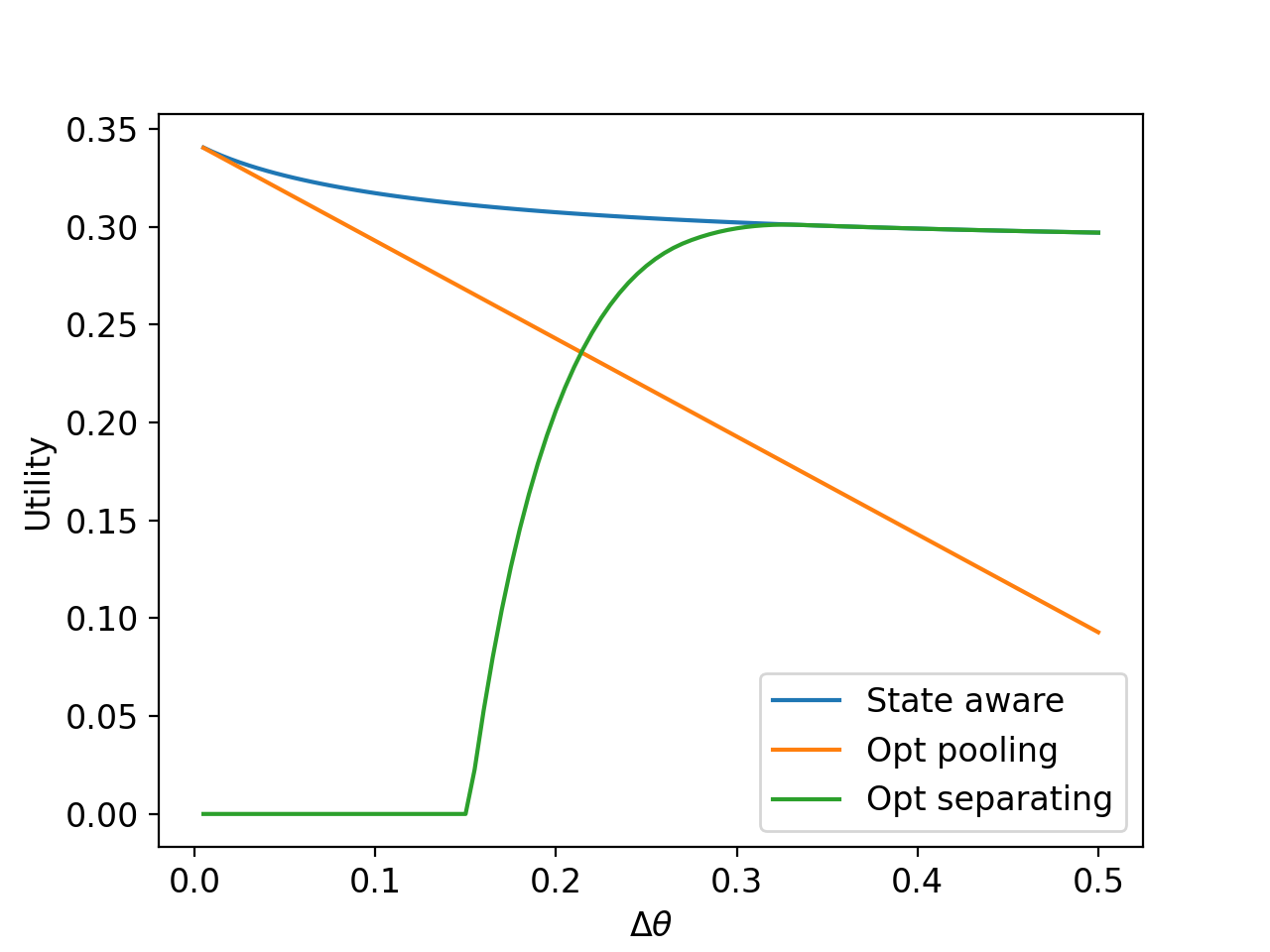}}
    \hfill
    \subfloat[{\small Utilities of state-aware and state-learning contracts for various $k$'s vs $\Delta \theta$  }\label{subfig:typeawareVsTypeLearnVaryK}]{%
      \includegraphics[width=0.3\textwidth]{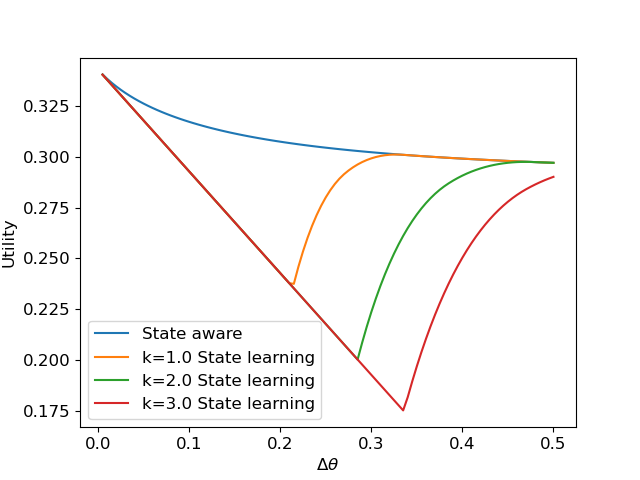}
    }
    \hfill 
     \centering
    \subfloat[{\small Worst-case sub-optimality of state-learning contracts relative to state-aware contracts}\label{subfig:typeLearnApprox}]{%
      \includegraphics[width=0.3\textwidth]{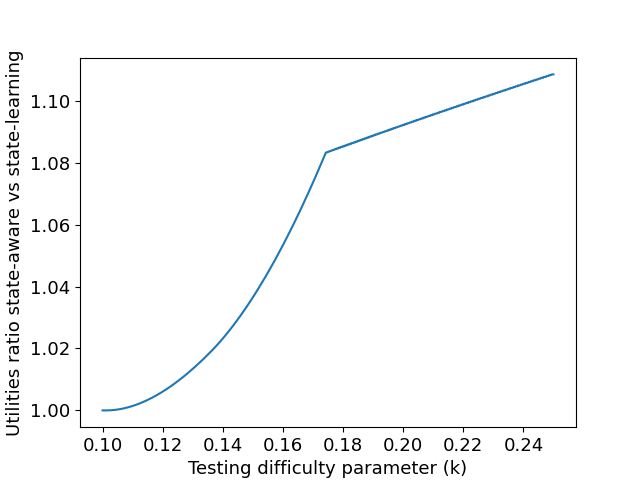}
    }
    \caption{Figure~\ref{subfig:typeAwarePoolSep} plots the utilities of the state-aware, separating, and the pooling contract against $(\Delta \theta)$. Figure~\ref{subfig:typeawareVsTypeLearnVaryK} again plots the utilities of contracts on the $y$-axis and $\Delta \theta$ on the $x$-axis. It plots the state-aware utility and the utilities of state-learning contracts of different levels $k$ of agent's testing efficiency. Figure~\ref{subfig:typeLearnApprox} plots the worst-case sub-optimality of state-learning contracts against $k$. The sub-optimality is the ratio of the state-learning contract's utility and the state-aware utility. The worst-case sub-optimality is the largest sub-optimality over all $\Delta \theta \in [0,0.5]$.} 
    \label{fig:typeLearningContract}
  \end{figure}

  \begin{itemize}[itemsep=-4pt, topsep=-4pt, leftmargin = *]
    \item Figure~\ref{subfig:typeAwarePoolSep} shows that the state-learning contract is pooling when $\Delta \theta$ is less than some threshold and is separating otherwise. For small and large values of $\Delta \theta$, the state-learning contract has utility close to the state-aware utility. 
    \item Figure~\ref{subfig:typeawareVsTypeLearnVaryK} shows that when it is more difficult to distinguish between $\theta_1, \theta_2$, the pooling contract is better than the separating contract for more values of $\Delta \theta$.
    \item Figure~\ref{subfig:typeLearnApprox} shows that the worst-case sub-optimality over all $\Delta \theta$ values of the state-learning contract compared to the state-aware utility increases as $k$ increases. When the agent can test more efficiently, the state-learning contract has greater utility for the principal. 
\end{itemize}

\section{Multi-round Delegation}\label{sec:multiround}
So far, we analyzed delegated learning that occurs through a single round of interaction between the principal and the agent. However, delegation often occurs over multiple rounds to allow the agent to learn more about the principal's requirements. Here we model such a scenario and analyze what happens when the principal uses a linear contract in each round. We introduce a notion of regret and show that repeated linear contracting over $T$ rounds results in $\Theta(T^{3/4})$ regret for the principal which is worse than the $O(T^{2/3})$ regret achievable without delegation i.e., the \emph{first-best regret}. 
We provide proof sketches of our results in the main body and provide the full proofs in the appendices.

\paragraph{The model.} 
To model uncertainty about the principal's requirements, we assume that the target distribution $D^*$ belongs to a class $\mathcal{D} = \{D_1, \ldots, D_k\}$. The agent knows the class $\mathcal{D}$ but does not know $D^*$ apriori. The principal contracts and deploys classifiers for $T$ rounds. 

\paragraph{Contracting Protocol.} For each round $i = 1, \ldots, T$:
    \begin{enumerate}[itemsep=0pt, topsep=0pt, leftmargin = *]
        \item The principal announces payment rule  $\rho_i$ which is a randomized mapping from a classifier to a positive, real-valued payment to the agent. 
        \item The agent chooses a target number of i.i.d.\ samples to collect from each distribution. Denote this number by $\mathbf{n_i} = (n_i^{(1)}, \ldots, n_i^{(k)})$ where $n_i^{(j)}$ is the number of samples the agent draws from distribution $D_j$ in round $i$. This choice is not observed by the principal.
        \item The agent provides classifier $h_i$ to the principal.
        \item The principal deploys a classifier $\bar{h}_i$. 
        \item The principal pays $\rho_i(h_i)$ to the agent according to the announced payment rule $\rho_i$.
    \end{enumerate}
The principal's and agent's action in each round can be chosen adaptively depending on the actions and outcomes of previous rounds. 
    
\paragraph{Utilities.} Over these $T$ rounds, the agent's utility is the sum of payments minus sample collection costs: $\sum_{i=1}^T \left (\rho_i(h_i) - \alpha \sum_{j=1}^k n_i^{(k)} \right )$. The principal's utility is the sum of accuracies minus payments: $\sum_{i=1}^T \left (1 - L_{D^*}(\bar{h}_i) - \beta \rho_i(h_i) \right)$.

\paragraph{Test-accuracy-based payments.} Our results deal with contracts based on test accuracy of the deployed classifier. In round $i \in [T]$ where the principal deploys the classifier $\bar{h}_i$, the payment is a function of the test accuracy $1 - L_{D^*}(\bar{h}_i) + \eta_i$, where $\eta_i$ is a mean-zero, random variable resulting in the principal's randomness of testing. A linear contract $c_i$ in this round offers payment $c_i$ times this test accuracy in this round.

\paragraph{Feedback-providing contracts.} In these repeated interactions, the payments serve as feedback to the agent to learn the principal's requirements as long as the principal's deployed classifiers depend on the classifiers provided by the agent. We focus on \emph{feedback-providing} contracts which satisfy the following property. A contract is feedback-providing if for every round $t$ with $c_t > 0$, the principal deploys the agent-provided classifier, and the payment for the round is therefore $c_t$ times the test accuracy of the agent-provided classifier.

We prove positive results on the utility the principal can achieve when contracting with a rational agent.

 \paragraph{$\mathcal{H}$-regret.} We introduce a notion of regret for the principal and the agent in this online setting. The regret notion is defined relative to a class of classifiers $\mathcal{H}$. It compares utility to the utility obtained by deploying the best classifier in $\mathcal{H}$ in every round, without any sample collection or payments. 

\begin{definition}[$\mathcal{H}$-regret]
    Let $\mathcal{H}$ be any class of classifiers. Let $((\mathbf{n}_t, h_t, \rho_t(h_t))_{t=1}^T$ be the sequence of actions by the principal and agent. 
    The principal's $\mathcal{H}$-regret ($R^P_T(\mathcal{H})$) is the difference in the utility of the sequence and the utility of deploying the most accurate classifier in $\mathcal{H}$ without payments: 
    \begin{align*}
        R^P_T(\mathcal{H}) &= T \max_{h \in \mathcal{H}} (1 - L_{D^*}(h)) - \sum_{t=1}^T \fParantheses{ 1 - L_{D^*}(h_t) - \EE \fBrackets{\rho_t(h_t)}}
    \end{align*}

    The agent's $\mathcal{H}$-regret ($R^A_T(\mathcal{H})$) is the difference in the utility of the sequence and the utility of deploying the highest expected payment yielding classifier in $\mathcal{H}$ and collecting no samples: 
    \begin{align*}
        R^A_T(\mathcal{H}) &= \sum_{t=1}^T \max_{h \in \mathcal{H}} \EE \fBrackets{\rho_t(h)} - \sum_{t=1}^T \fParantheses{ \rho_t(h_t) -\alpha \sum_{j=1}^k n_t^{(j)}}.
    \end{align*}
\end{definition}

We consider a model of rationality for the agent in repeated linear contracting. This is an assumption on the agent's $\Hcal$-Regret rate. 

\paragraph{Agent's rationality.} Let $(c_1)_{t=t}^T$ be the sequence of linear contracts employed by the principal. A rational agent achieves $\Hcal$-regret sub-linear in $\sum_{t=1}^T c_t$. 

We discuss the implications of contracting with a rational agent first and then justify the rationality assumption by providing an algorithm for the agent to achieve the rationality criteria.  

In the following proposition, we show how contracting with such a rational agent provides the principal a sub-linear $\Hcal$-regret. The proof of the proposition constructs a contract that the principal can use to achieve sub-linear $\mathcal{H}$-regret. This contract makes use of an upper bound on the agent's regret rate.  

 \begin{proposition} \label{prop:multiRoundPrincipalRegret}
    Suppose the agent achieves $O\fParantheses{\fParantheses{\sum_{t=1}^T c_t}^x}$ agent's $\Hcal$-regret with $x < 1$, for any sequence of contracts $(c_t)_{t=1}^T$. Then, the principal can achieve $O\fParantheses{T^{\frac 1 {2-x}}}$ $\Hcal$-regret.
\end{proposition} 

\begin{proof}
The contract that allows this regret for the principal is described here. The principal contracts with $c_t = 1$ for the first $N = T^{1/(2-x)}$ rounds. In these rounds $t \in [1, N]$, the principal deploys the agent-provided classifier $h_t$. The principal shuts down contracting for the remaining rounds. That is, $c_t = 0$ for $t \in (N, T]$. In rounds $t \in (N,T]$, the principal deploys a classifier $\bar{h}$ that is uniformly, at random picked from the classifiers $h_1, \ldots, h_N$ deployed in the first $N$ rounds. 

In the first $N$ rounds, the principal's $\mathcal{H}$-regret is bounded by $O(N)$ trivially. In the remaining $T-N$ rounds, the principal's $\mathcal{H}$-regret is bounded above by just the suboptimality of the deployed classifiers since no payments are offered. The regret in the last $T-N$ rounds is at most
\begin{align*}
    (T - n) \left ( \E[L_{D^*}(\bar{h})] - \theta^* \right ) &= (T - n) \frac 1 N \sum_{i \in [N]}\left ( \E[L_{D^*}(\bar{h})] - \theta^* \right ) \\
    &\leq (T - n) \frac 1 N \cdot \text{Agent's $\mathcal{H}$-regret} \\
    &\in O \left (\frac{T}{N^{1-x}}\right ),
\end{align*}
where the last step follows from the rationality assumption that the agent achieves a $\mathcal{H}$-regret of $O(N^x)$.
The cumulative principal's $\mathcal{H}$-regret overall all $T$ rounds is therefore $O(N + T / N^{1-x})$. Since $N = T^{1/(2-x)}$, principal's $\mathcal{H}$-regret is $O(T^{1/(2-x)})$.
 
\end{proof}

The above proposition shows how the principal can achieve sublinear $\mathcal{H}$-regret under the agent's rationality assumption. Now we show that the agent's rationality can be achieved. In our proof, we provide an algorithm for the agent to achieve $\mathcal{H}$-regret sublinear in $\left ( \sum_{t \in [T]} c_t \right )^{2/3}$.
\begin{proposition}\label{prop:agentAlgorithmRepeated}
    For any sequence of feedback-providing contracts $(c_t)_{t=1}^T$, the agent can achieve $O \left (\left ( \sum_{t \in [T]} c_t \right )^{2/3} \right )$ $\mathcal{H}$-regret.
\end{proposition}
\begin{proof}[Proof sketch]
    We will describe the agent's algorithm here and defer the full analysis of this algorithm to the appendices. For any round $t$, let us denote $\sum_{t' \leq t} c_t$ by $s_t$. Since the agent incurs cost linear in the number of samples collected, the algorithm must limit the number of samples to $O \left (s_T^{2/3} \right )$. To do this, at round $t$, the agents collects $\max \left (0, \left \lfloor s_{t}^{2/3} \right \rfloor - \left \lfloor s_{t-1}^{2/3} \right \rfloor \right )$  samples. As a result, for any $t$, the number of samples collected up to round $t$ is $\lfloor s_t^{2/3} \rfloor$.To select the distribution to sample from, the agent picks one out of $\{D_1, \ldots, D_k\}$ uniformly at random. This fully describes how the agent's algorithm collects samples by describing the number of samples and distributions to sample from.

    The remaining component of the agent's algorithm is classifier selection. The classifier selection includes exploration and exploitation where the exploration is to determine which of $D_1, \ldots, D_k$ provides samples to train an accurate classifier. Since we know $D^*$ is one of these $k$ distributions, we are guaranteed that there exists a distribution providing relevant samples. 
    
    We assign rounds of sample collection to be rounds of exploration. These rounds are suitable for exploration for two reasons. The first is because the principal deploys the agent-provided classifiers in these rounds. By the algorithm's design, the agent only collects samples in rounds with $c_t > 0$, and in feedback-providing contracts, agent-provided classifiers are deployed in these rounds. Secondly, this choice also leads to the right number of exploration rounds for the optimal exploration-exploitation tradeoff. Let $\{i_1, \ldots, i_r\}$ denote the indices of rounds of exploration. This set's size is at most $\left ( \sum_{t \in [T]} c_t \right )^{2/3}$. A phased exploration is done in these rounds. 

    The phased exploration divides the indices $\{i_1, \ldots, i_r\}$ into phases where phase $j$ is of length $2^{j-1}$. So at least $2^{j-1}$ samples are collected in phase $j$. The number of phases is therefore $\log \left (\left ( \sum_{t \in [T]} c_t \right )^{2/3} \right )$. 
    
    Each phase $j$ uniformly explores classifiers learned based on samples from each distribution, collected up to phase $j-1$. Let $Q^{(j)}_i$ be the set of samples collected from $D_i$ in phases $1, \ldots, j-1$. And let $h^{(j)}_i = \mathrm{ERM}_{\mathcal{H}}\left ( Q^{(j)}_i \right)$. In each round of phase $j$, the agent picks one of $\{h^{(j)}_1, \ldots, h^{(j)}_k\}$ uniformly at random. Let $h^*_j$ be the classifier of phase $j$ with the highest sum of test accuracies. The agent can compute the test accuracy by dividing payment by the contract coefficient (which is non-zero in exploration phases). 

    The agent treats all rounds other than $\{i_1, \ldots, i_r\}$ as exploitation rounds. For an exploitation round $t$, suppose the last collected sample was in phase $j$. Then in round $t$, the agent selects the best classifier in phase $j-1$: $h^*_{j-1}$. Due to the completion of the $(j-1)^{\text{th}}$ exploration phase, $h^*_{j-1}$ is guaranteed to have some optimality guarantees. 

    The regret analysis for this algorithm is presented in Appendix~\ref{proof:agentAlgorithmRepeated}.

\end{proof}
Propositions~\ref{prop:multiRoundPrincipalRegret} and~\ref{prop:agentAlgorithmRepeated} together imply that the principal can achieve $O(T^{3/4})$ $\mathcal{H}$-Regret through linear contracting with a rational agent. We next show that the principal cannot achieve better $\mathcal{H}$-regret rates when the agent's $\mathcal{H}$-regret rate is $O \left (\left (\sum_{t \in [T]} c_t \right )^{2/3} \right )$.

\begin{proposition}
    For any sequence of linear contracts $(c_t)_{t\in [T]}$, if the principal contracts with an agent with $O \left (\left (\sum_{t \in [T]} c_t \right )^{2/3} \right )$ $\mathcal{H}$-regret on all problem instances, the principal's $\mathcal{H}$-regret is $\Omega(T^{3/4})$ for some problem instance.
\end{proposition}
\begin{proof}
 Let us denote $\sum_{t=1}^T c_t$ by $s_T$. An agent with $O \left (s_T^{x} \right )$ $\mathcal{H}$-regret collects $O \left (s_T^{x} \right )$ samples since the agent incurs cost linear in the number of samples collected. This upper bound on the number of samples provides a lower bound on the excess errors of the deployed classifiers and therefore a lower bound on the principal's regret. 

    Usual sample complexity lower bounds provide lower bounds on the error of learning using a number of i.i.d samples. However, in our setting, the agent has more than just the samples he collects to learn classifiers. Through the linear payments he receives, he also has access to estimates of the accuracy of classifiers he provides in each round. This is a form of (noisy) query access to the distribution.  

We provide a min-max lower bound on learning using both i.i.d samples and queries of the form answering the expected error on $D^*$ of a classifier in the following proposition. We show that in a min-max sense, the queries do not allow for more accurate classifiers compared to using just the i.i.d samples. 
 
\begin{lemma}[Lower bound on error from using samples and queries]\label{prop:queriesuseless}
Consider a learning algorithm that uses $m$ i.i.d samples and $q$ queries of accuracies of classifiers. Then there exists a distribution $D$ for which the expected error of the learned classifier is $\Omega(1/\sqrt{m})$ more than the optimal error of a one-dimensional halfspace on $D$. 
\end{lemma}
We prove this lemma in Appendix~\ref{proof:queriesuseless}. The main intuition for this lower bound is that without any structure on the distribution, it is hard to know what classifiers are useful to query. Therefore the queries are not useful. 

Due to the above proposition, there is a problem instance for which the classifiers the principal deploys have an average excess error of $\Omega(1/s_T^{1/3})$ resulting in regret $\Omega(T / s_T^{1/3})$. 

The principal also incurs a cost of order $\Omega(s_T)$ due to the contracting. Due to the sub-linearity of agent's regret in $s_T$, the payments have to be $\Omega(s_T)$ resulting in this term of the principal's regret.

The optimal choice of $s_T$ to balance these two terms is $s_T \in \Theta (T^{3/4})$ resulting in regret of order $\Omega(T^{3/4})$ .

% \vspace{-0.2cm}
\end{proof}

\textbf{Comparison with non-delegation benchmark.} We have shown that the principal's $\mathcal{H}$-regret through delegating with an $O \left (\left (\sum_{t \in [T]} c_t \right )^{2/3} \right )$ $\mathcal{H}$-regret agent is $\Omega(T^{3/4})$. Since there is an algorithm for the agent to achieve $O \left (\left (\sum_{t \in [T]} c_t \right )^{2/3} \right )$ $\mathcal{H}$-regret, a rational agent is likely to achieve this regret upper bound. 

If the principal instead performs this learning without delegation, the achievable regret is the same as the agent's achievable regret when the contract coefficient is $1$ in every round. By Proposition~\ref{prop:agentAlgorithmRepeated}, $O(T^{2/3})$ regret is achievable. Note that this means the $\mathcal{H}$-regret rate is strictly worse for the principal through delegation compared to without delegation.
  
\section{Conclusions}
Different parts of the learning pipeline are increasingly being delegated to autoML services and firms. Focusing on the delegation of data collection in such settings, we developed a principal-agent model capturing the practical challenges of hidden action and hidden state arising due to information asymmetry.  We hope our work inspires future research exploring and addressing other incentive-theoretic obstacles that arise in this domain and in the delegation of other parts of the learning pipeline.

We identified the practicality of linear contracts under many problem parameters. We also obtained insights that hold for many other delegation settings, including the decreased utility and information rent the principal faces due to not knowing the hidden state. We also addressed a more realistic form of hidden state information asymmetry where the agent gradually learns the hidden state while executing the contract. 

Many of our results rely on a specific structure for the dependence of error rates on number of samples used for training. These rates are exact for some learning tasks (see Remark~\ref{rem:linregression}) but are generally upper bounds. From the principal's perspective, the contracts we designed continue to have good accuracy guarantees even if the error rates are just upper bounds and not exact. The agent's perspective is more complicated. Since the agent can learn more about the true error curve during learning and not rely on the upper bound, analyzing how the error curve learning occurs is needed to allow us to design truly incentive-compatible contracts in this more general setting. However, learning the error curve shape is learning from a much broader class and is likely to be more challenging. How this challenge impacts the utility of contracts would be an interesting direction for future work.

\section*{Acknowledgements}
Funded in part by the European Union (ERC-2022-SYG-OCEAN-101071601). Views and opinions expressed
are however those of the authors only and do not necessarily reflect those of the European
Union or the European Research Council.

This work was supported in part by the National Science Foundation under grant CCF-2145898, by the Office of Naval Research under grant N00014-24-1-2159, a C3.AI Digital Transformation Institute grant, and Alfred P. Sloan fellowship, and a Schmidt Science AI2050 fellowship

\nocite{*}

\bibliography{bibliography}

\begin{thebibliography}{28}
\providecommand{\natexlab}[1]{#1}
\providecommand{\url}[1]{\texttt{#1}}
\expandafter\ifx\csname urlstyle\endcsname\relax
  \providecommand{\doi}[1]{doi: #1}\else
  \providecommand{\doi}{doi: \begingroup \urlstyle{rm}\Url}\fi

\bibitem[Acemoglu et~al.(2022)Acemoglu, Makhdoumi, Malekian, and Ozdaglar]{acemoglu2022too}
Daron Acemoglu, Ali Makhdoumi, Azarakhsh Malekian, and Asu Ozdaglar.
\newblock Too much data: Prices and inefficiencies in data markets.
\newblock \emph{American Economic Journal: Microeconomics}, 14\penalty0 (4):\penalty0 218--256, 2022.

\bibitem[Agarwal et~al.(2019)Agarwal, Dahleh, and Sarkar]{agarwal2019marketplace}
Anish Agarwal, Munther Dahleh, and Tuhin Sarkar.
\newblock A marketplace for data: An algorithmic solution.
\newblock In \emph{Proceedings of the 2019 ACM Conference on Economics and Computation}, pages 701--726, 2019.

\bibitem[Agarwal et~al.(2020)Agarwal, Dahleh, Horel, and Rui]{agarwal2020towards}
Anish Agarwal, Munther Dahleh, Thibaut Horel, and Maryann Rui.
\newblock Towards data auctions with externalities.
\newblock \emph{arXiv preprint arXiv:2003.08345}, 2020.

\bibitem[Alon et~al.(2021)Alon, D{\"u}tting, and Talgam-Cohen]{alon2021contracts}
Tal Alon, Paul D{\"u}tting, and Inbal Talgam-Cohen.
\newblock Contracts with private cost per unit-of-effort.
\newblock In \emph{Proceedings of the 22nd ACM Conference on Economics and Computation}, pages 52--69, 2021.

\bibitem[Alon et~al.(2022)Alon, D{\"u}tting, Li, and Talgam-Cohen]{alon2022bayesian}
Tal Alon, Paul D{\"u}tting, Yingkai Li, and Inbal Talgam-Cohen.
\newblock Bayesian analysis of linear contracts.
\newblock \emph{arXiv preprint arXiv:2211.06850}, 2022.

\bibitem[Bates et~al.(2022)Bates, Jordan, Sklar, and Soloff]{bates2022principalagent}
Stephen Bates, Michael~I. Jordan, Michael Sklar, and Jake~A. Soloff.
\newblock Principal-agent hypothesis testing.
\newblock \emph{arXiv preprint arXiv:2205.06812}, 2022.

\bibitem[Bechtel et~al.(2022)Bechtel, Dughmi, and Patel]{bechtel2022delegated}
Curtis Bechtel, Shaddin Dughmi, and Neel Patel.
\newblock Delegated {P}andora's box.
\newblock \emph{arXiv preprint arXiv:2202.10382}, 2022.

\bibitem[Bergemann and Bonatti(2019)]{bergemann2019markets}
Dirk Bergemann and Alessandro Bonatti.
\newblock Markets for information: An introduction.
\newblock \emph{Annual Review of Economics}, 11:\penalty0 85--107, 2019.

\bibitem[Bolton and Dewatripont(2004)]{bolton2004contract}
Patrick Bolton and Mathias Dewatripont.
\newblock \emph{Contract Theory}.
\newblock 2004.

\bibitem[Cai et~al.(2015)Cai, Daskalakis, and Papadimitriou]{cai2015optimum}
Yang Cai, Constantinos Daskalakis, and Christos Papadimitriou.
\newblock Optimum statistical estimation with strategic data sources.
\newblock In \emph{Conference on Learning Theory}, pages 280--296. PMLR, 2015.

\bibitem[Carroll(2015)]{carroll2015robustness}
Gabriel Carroll.
\newblock Robustness and linear contracts.
\newblock \emph{American Economic Review}, 105\penalty0 (2):\penalty0 536--63, 2015.

\bibitem[Chade and Swinkels(2019)]{chade2019disentangling}
Hector Chade and Jeroen Swinkels.
\newblock Disentangling moral hazard and adverse selection.
\newblock Technical report, Working Paper, Arizona State University.[450], 2019.

\bibitem[Chen et~al.(2022)Chen, Li, and Xu]{chen2022selling}
Junjie Chen, Minming Li, and Haifeng Xu.
\newblock Selling data to a machine learner: Pricing via costly signaling.
\newblock In \emph{International Conference on Machine Learning}, pages 3336--3359. PMLR, 2022.

\bibitem[Chen et~al.(2019)Chen, Koutris, and Kumar]{chen2019towards}
Lingjiao Chen, Paraschos Koutris, and Arun Kumar.
\newblock Towards model-based pricing for machine learning in a data marketplace.
\newblock In \emph{Proceedings of the 2019 International Conference on Management of Data}, pages 1535--1552, 2019.

\bibitem[Chiesa and Gur(2018)]{chiesa2018proofs}
Alessandro Chiesa and Tom Gur.
\newblock Proofs of proximity for distribution testing.
\newblock In \emph{9th Innovations in Theoretical Computer Science Conference (ITCS 2018)}. Schloss Dagstuhl-Leibniz-Zentrum fuer Informatik, 2018.

\bibitem[D{\"u}tting et~al.(2019)D{\"u}tting, Roughgarden, and Talgam-Cohen]{dutting2019simple}
Paul D{\"u}tting, Tim Roughgarden, and Inbal Talgam-Cohen.
\newblock Simple versus optimal contracts.
\newblock In \emph{Proceedings of the 2019 ACM Conference on Economics and Computation}, pages 369--387, 2019.

\bibitem[D\"{u}tting et~al.(2020)D\"{u}tting, Roughgarden, and Cohen]{dutting2020complexity}
Paul D\"{u}tting, Tim Roughgarden, and Inbal-Talgam Cohen.
\newblock The complexity of contracts.
\newblock In \emph{Proceedings of the 2020 {ACM}-{SIAM} {S}ymposium on {D}iscrete {A}lgorithms}, pages 2688--2707. SIAM, Philadelphia, PA, 2020.

\bibitem[Frazier et~al.(2014)Frazier, Kempe, Kleinberg, and Kleinberg]{frazier2014incentivizing}
Peter Frazier, David Kempe, Jon Kleinberg, and Robert Kleinberg.
\newblock Incentivizing exploration.
\newblock In \emph{Proceedings of the Fifteenth ACM Conference on Economics and Computation}, pages 5--22, 2014.

\bibitem[Goldwasser et~al.(2021)Goldwasser, Rothblum, Shafer, and Yehudayoff]{goldwasser2021interactive}
Shafi Goldwasser, Guy~N Rothblum, Jonathan Shafer, and Amir Yehudayoff.
\newblock Interactive proofs for verifying machine learning.
\newblock In \emph{12th Innovations in Theoretical Computer Science Conference (ITCS 2021)}. Schloss Dagstuhl-Leibniz-Zentrum f{\"u}r Informatik, 2021.

\bibitem[Guruganesh et~al.(2021)Guruganesh, Schneider, and Wang]{guruganesh2021contracts}
Guru Guruganesh, Jon Schneider, and Joshua~R Wang.
\newblock Contracts under moral hazard and adverse selection.
\newblock In \emph{Proceedings of the 22nd ACM Conference on Economics and Computation}, pages 563--582, 2021.

\bibitem[Ho et~al.(2016)Ho, Slivkins, and Vaughan]{ho2016adaptive}
Chien-Ju Ho, Aleksandrs Slivkins, and Jennifer~Wortman Vaughan.
\newblock Adaptive contract design for crowdsourcing markets: Bandit algorithms for repeated principal-agent problems.
\newblock \emph{Journal of Artificial Intelligence Research}, 55:\penalty0 317--359, 2016.

\bibitem[Jia et~al.(2021)Jia, Yaghini, Choquette-Choo, Dullerud, Thudi, Chandrasekaran, and Papernot]{jia2021proof}
Hengrui Jia, Mohammad Yaghini, Christopher~A Choquette-Choo, Natalie Dullerud, Anvith Thudi, Varun Chandrasekaran, and Nicolas Papernot.
\newblock Proof-of-learning: Definitions and practice.
\newblock In \emph{2021 IEEE Symposium on Security and Privacy (SP)}, pages 1039--1056. IEEE, 2021.

\bibitem[Kleinberg and Kleinberg(2018)]{kleinberg2018delegated}
Jon Kleinberg and Robert Kleinberg.
\newblock Delegated search approximates efficient search.
\newblock In \emph{Proceedings of the 2018 ACM Conference on Economics and Computation}, pages 287--302, 2018.

\bibitem[Laffont and Martimort(2009)]{laffont2009theory}
Jean-Jacques Laffont and David Martimort.
\newblock \emph{The Theory of Incentives}.
\newblock Princeton University Press, 2009.

\bibitem[Liu et~al.(2020)Liu, Wang, Shen, Liu, and Chen]{liu2020incentivized}
Zhiyuan Liu, Huazheng Wang, Fan Shen, Kai Liu, and Lijun Chen.
\newblock Incentivized exploration for multi-armed bandits under reward drift.
\newblock In \emph{Proceedings of the AAAI Conference on Artificial Intelligence}, volume~34, pages 4981--4988, 2020.

\bibitem[Rolf et~al.(2021)Rolf, Worledge, Recht, and Jordan]{rolf2021representation}
Esther Rolf, Theodora~T Worledge, Benjamin Recht, and Michael Jordan.
\newblock Representation matters: Assessing the importance of subgroup allocations in training data.
\newblock In \emph{International Conference on Machine Learning}, pages 9040--9051. PMLR, 2021.

\bibitem[Saig et~al.(2023)Saig, Talgam-Cohen, and Rosenfeld]{saig2023delegated}
Eden Saig, Inbal Talgam-Cohen, and Nir Rosenfeld.
\newblock Delegated classification.
\newblock \emph{arXiv preprint arXiv:2306.11475}, 2023.

\bibitem[Ying et~al.(2019)Ying, Klein, Christiansen, Real, Murphy, and Hutter]{ying2019bench}
Chris Ying, Aaron Klein, Eric Christiansen, Esteban Real, Kevin Murphy, and Frank Hutter.
\newblock Nas-bench-101: Towards reproducible neural architecture search.
\newblock In \emph{International Conference on Machine Learning}, pages 7105--7114. PMLR, 2019.

\end{thebibliography}
\bibliographystyle{plainnat}

\newpage

\appendix

\section{Omitted proofs}\label{app:proofs}

\subsection{Proof of Proposition~\ref{thm:LinContractApprox}}
\begin{proof}
The linear contract $c^*$ that achieves this approximately optimal utility is the following:
\[c^* = \max \fParantheses{\frac 1 {\beta(p+1)^{\frac{p+1}{p}}}, \frac{\alpha d^{\frac 1 p}}{p} \cdot \fParantheses{\frac{p+1}{1-\theta}}^{\frac{p+1}{p}}}.\]

We first show that this contract satisfies the participation constraint for the agent. When the linear contract is $c$ times the accuracy, the agent the number of samples $n$ to maximize  the agent's utility $c\fParantheses{1 - \theta - \frac{d}{n^p}} - \alpha n$. The number of samples the agent chooses as a function of $c$ is $\fParantheses{\frac{cdp}{\alpha}}^{\frac{1}{p+1}}$. The contract $c$ satisfies the  participation constraint if the utility from choosing this number of samples is non-negative. This utility is:
\begin{align*}
    &=c \fParantheses{1 - \theta - d \fParantheses{\frac{\alpha}{cdp}}^{\frac{p}{p+1}}} - \alpha \fParantheses{\frac{cdp}{\alpha}}^\frac{1}{p+1} \\
    &= c(1-\theta) - c^{\frac{1}{p+1}} \fParantheses{\frac{\alpha d^{\frac 1 p}}{p}}^{\frac{p}{p+1}} - c^{\frac{1}{p+1}} \fParantheses{{\alpha d^{\frac 1 p}}^{\frac{p}{p+1}}} p^{\frac{1}{p+1}} \\
    &= c(1-\theta) - c^{\frac{1}{p+1}} \cdot \frac{p+1}{1-\theta} \cdot \fParantheses{\frac{\alpha d^{\frac{1}{p}}}{p}}^\frac{p}{p+1}.
\end{align*}
This utility is non-negative when $c \geq \frac{\alpha d^{\frac{1}{p}}}{p} \fParantheses{\frac{p+1}{1-\theta}}^\frac{p+1}{p}$. By the definition of $c^*$, it is greater than $\frac{\alpha d^{\frac{1}{p}}}{p}$ and so $c^*$ satisfies the participation constraint.

When the principal chooses a linear contract $c$, it achieves a utility
\begin{align*}
    U_{\text{lin}}(c) &= (1 - \beta c) \fParantheses{1 - \theta - \fParantheses{\frac{\alpha d^{\frac 1 p}}{pc}}^\frac{p}{p+1}}.
\end{align*}
We can provide an upper bound on the optimum utility using the optimum utility of the principal when there is no noise in the observed accuracy. In this case, the principal gets the agent to collect $\fParantheses{\frac{dp}{\alpha \beta }}^{\frac{1}{p+1}}$ and pays the agent $\alpha$ times this amount. So the optimum utility $U_{\text{opt}}$ satisfies 
\[U_{\text{opt}} \leq 1 - \theta - (p+1) \fParantheses{\frac{\alpha\beta d^{\frac 1 p}}{p}}^\frac{p}{p+1}.\]

To show that $c^*$ achieves the approximation guaranteed in the theorem, we consider two cases. 

\paragraph{Case 1.} The first case is when $c^* = \frac 1 {\beta(p+1)^{\frac{p+1}{p}}}$. In this case the utility of $c^*$ is 
\begin{align*}
    U_{\text{lin}}(c^*) &= \fParantheses{1 - \frac{1}{(p+1)^\frac{p+1}{p}}} \fParantheses{1 - \theta - (p+1) \fParantheses{\frac{\alpha\beta d^{\frac 1 p}}{p}}^\frac{p}{p+1}} \\
    &\geq \fParantheses{1 - \frac{1}{(p+1)^\frac{p+1}{p}}} U_{\textbf{opt}}
\end{align*}
So in the first case, we have shown the required bound on the approximation ratio of the contract $c^*$.

\paragraph{Case 2.}The other case is when $c^* = \frac{\alpha d^{\frac 1 p}}{p}  \fParantheses{\frac{p+1}{1-\theta}}^{\frac{p+1}{p}}$. 
In this case, 
\begin{align*}
    U_{\text{lin}}(c^*) &= \fParantheses{1 - \frac{\alpha\beta d^{\frac 1 p}}{p} \fParantheses{\frac{p+1}{1-\theta}}^\frac{p+1}{p}} \cdot {\frac{(1-\theta)p}{p+1}} \\
    \frac{U_{\text{lin}}(c^*)}{U_{\text{opt}}} &= \frac{\fParantheses{1 - \frac{\alpha\beta d^{\frac 1 p}}{p} \fParantheses{\frac{p+1}{1-\theta}}^\frac{p+1}{p}} \cdot {\frac{(1-\theta)p}{p+1}}}{(1 - \theta) \left ( 1 - \frac{p+1}{1-\theta} \fParantheses{\frac{\alpha\beta d^{\frac 1 p}}{p}}^\frac{p}{p+1} \right )} \\
    &= \frac{\frac{p}{p+1} \fParantheses{1-t^\frac{p+1}{p}}}{1-t} &{\left (\text{where } t = \frac{p+1}{1-\theta} \cdot \fParantheses{\frac{\alpha\beta d^{\frac 1 p}}{p}}^\frac{p}{p+1} \right ).} \\
\end{align*}

${U_{\text{lin}}(c^*)} / {U_{\text{opt}}}$ is increasing in $t$. The condition of this case 2 occurring turns out to be a condition on $t$. This condition for case 2 occurring is the following:
\begin{align*}
    \frac 1 {\beta(p+1)^{\frac{p+1}{p}}} &\leq \frac{\alpha d^{\frac 1 p}}{p}  \fParantheses{\frac{p+1}{1-\theta}}^{\frac{p+1}{p}} \\
    \Longrightarrow \frac 1 {p+1} &\leq \frac{p+1}{1-\theta} \cdot \fParantheses{\frac{\alpha\beta d^{\frac 1 p}}{p}}^\frac{p}{p+1} \\
    \Longrightarrow \frac{1}{p+1} &\leq t.
\end{align*}
Case 2 occurs when $t \geq 1/(p+1)$ and in this case ${U_{\text{lin}}(c^*)} / {U_{\text{opt}}}$ is increasing in $t$. So the smallest value of ${U_{\text{lin}}(c^*)} / {U_{\text{opt}}}$ in this case occurs when $t = 1/(p+1)$. Plugging in this value, we get that for this case, 
\[{\frac{U_{\text{lin}}(c^*)}{U_{\text{opt}}}} \geq {1 - \frac{1}{(p+1)^\frac{p+1}{p}}} \]
\end{proof}

\subsection{Proof of Theorem~\ref{thm:sampleCompNoMoralHazard}}
\begin{proof}
Recall that the first-best contract has a threshold form. The contract offers payment $t^*$ when the test error is less than or equal to $\ell^*$ and offers payment zero otherwise. Let us denote the sample complexity to get expected loss at most $\ell$ by $n(\ell)$. That is, 
\[n(\ell) = \fParantheses{\frac d {\ell - \theta}}^{1/p}.\] The optimal contract offers $t^* = \alpha n(\ell^*)$ where $\alpha$ is the cost per sample for the agent. Let us denote $n(\ell^*)$ by $n^*$. And $n^* = (pd/\alpha \beta)^{1/(p+1)}$.

We will show that the best response for the agent against this contract is never to collect samples less than $n(\ell + \epsilon)$ when the test set has size $O\left (\frac 1 {\epsilon^2} \log \frac d {\epsilon} \right)$. We show this by showing that the agent's utility in choosing $n(\ell+\Delta)$ is less than the agent's utility in collecting $n(\ell - \Delta)$ for all $\Delta > \epsilon$.

The number of samples the agent would collect to get expected error $l^* + \Delta$ is such that:
\begin{align*}
\theta + \frac d {n_1^p} &= \theta + \frac d {n^{*p}} + \Delta \\
n_1 &= \frac{n^*d^p}{(d+\Delta)^p}.
\end{align*}

Similarly, the number of samples needed to get expected error $l^* - \Delta$ is 
\[n_2 = \frac{n^*d^p}{(d-\Delta)^p}.\]

For any action of the agent, the probability that the observed loss is $\epsilon$ or more away from the expected loss is less than $2\exp(-2m \epsilon^2)$. This is by applying Hoeffding's inequality on the observed loss random variable which is bounded between 0 and 1. As a result, for $\Delta > \epsilon$, the expected payment when collecting $n_1$ and $n_2$ samples is $\leq 2t^* \exp (-2m\epsilon^2)$ and $\geq t^*(1 - 2\exp(-2m \epsilon^2))$ respectively.  The agent's utility due to $n_1$ is less than the utility due to $n_2$ when 
\begin{align*}
    \alpha n^* \fParantheses{1 - 4\exp(-2m\epsilon^2)} &\geq \alpha n^* d^p \fParantheses{\frac 1 {(d - 2\Delta n^{*1/p})^p} - \frac 1 {(d + 2\Delta n^{*1/p})^p}}.
    \intertext{Let us denote $\kappa  = d^p \fParantheses{\frac 1 {(d - 2\Delta n^{*1/p})^p} - \frac 1 {(d + 2\Delta n^{*1/p})^p}}$. So this occurs when}
    m &\geq \frac 1 {2 \epsilon^2} \log \frac 4 {1-\kappa}.
\end{align*}
Note that $\frac 1 {1-\kappa}$ is polynomial in both $d$ and $\frac 1 \epsilon$.

\end{proof}

\subsection{Proof of Proposition~\ref{prop:agentAlgorithmRepeated}}\label{proof:agentAlgorithmRepeated}
\begin{proof}
We start by restating the agent's algorithm. For any round $t$, let us denote $\sum_{t' \leq t} c_t$ by $s_t$. Since the agent incurs cost linear in the number of samples collected, the algorithm must limit the number of samples to $O \left (s_T^{2/3} \right )$. To do this, at round $t$, the agents collects $\max \left (0, \left \lfloor s_{t}^{2/3} \right \rfloor - \left \lfloor s_{t-1}^{2/3} \right \rfloor \right )$  samples. As a result, for any $t$, the number of samples collected up to round $t$ is $\lfloor s_t^{2/3} \rfloor$.To select the distribution to sample from, the agent picks one out of $\{D_1, \ldots, D_k\}$ uniformly at random. This fully describes how the agent's algorithm collects samples by describing the number of samples and distributions to sample from.

    The remaining component of the agent's algorithm is classifier selection. The classifier selection includes exploration and exploitation where the exploration is to determine which of $D_1, \ldots, D_k$ provides samples to train an accurate classifier. Since we know $D^*$ is one of these $k$ distributions, we are guaranteed that there exists a distribution providing relevant samples. 
    
    We assign rounds of sample collection to be rounds of exploration. These rounds are suitable for exploration for two reasons. The first is because the principal deploys the agent-provided classifiers in these rounds. By the algorithm's design, the agent only collects samples in rounds with $c_t > 0$, and in feedback-providing contracts, agent-provided classifiers are deployed in these rounds. Secondly, this choice also leads to the right number of exploration rounds for the optimal exploration-exploitation tradeoff. Let $\{i_1, \ldots, i_r\}$ denote the indices of rounds of exploration. This set's size is at most $\left ( \sum_{t \in [T]} c_t \right )^{2/3}$. A phased exploration is done in these rounds. 

    The phased exploration divides the indices $\{i_1, \ldots, i_r\}$ into phases where phase $j$ is of length $2^{j-1}$. So at least $2^{j-1}$ samples are collected in phase $j$. The number of phases is therefore $\log \left (\left ( \sum_{t \in [T]} c_t \right )^{2/3} \right )$. 
    
    Each phase $j$ uniformly explores classifiers learned based on samples from each distribution, collected up to phase $j-1$. Let $Q^{(j)}_i$ be the set of samples collected from $D_i$ in phases $1, \ldots, j-1$. And let $h^{(j)}_i = \mathrm{ERM}_{\mathcal{H}}\left ( Q^{(j)}_i \right)$. In each round of phase $j$, the agent picks one of $\{h^{(j)}_1, \ldots, h^{(j)}_k\}$ uniformly at random. Let $h^*_j$ be the classifier of phase $j$ with the highest sum of test accuracies. The agent can compute the test accuracy by dividing payment by the contract coefficient (which is non-zero in exploration phases).

    The agent treats all rounds other than $\{i_1, \ldots, i_r\}$ as exploitation rounds. For an exploitation round $t$, suppose the last collected sample was in phase $j$. Then in round $t$, the agent selects the best classifier in phase $j-1$: $h^*_{j-1}$. Due to the completion of the $(j-1)^{\text{th}}$ exploration phase, $h^*_{j-1}$ is guaranteed to have some optimality guarantees. 
    
\paragraph{Regret analysis.} Now we analyze the regret incurred due to this algorithm. The regret due to cost of sampling is $O(s_T^{2/3})$. We will analyze the regret due to the sub-optimality of the classifiers deployed. Let $\theta^* = \min_{h \in \mathcal{H}} L_{D^*}(h)$. Since the number of exploration rounds is $O(s_T^{2/3})$, the regret in these rounds is trivially bounded by $O(s_T^{2/3})$. We will focus on sub-optimality of classifiers deployed in the exploitation rounds. 
 
 Consider the classifier $h^*_j$ yielding the highest accuracy in phase $j$. Let $L_j(h)$ indicate the average test loss of deploying classifier $h$ in phase $j$. Each of the candidate classifiers of phase $j$: $\{h^{(j)}_1, \ldots, h^{(j)}_k\}$ is explored $2^{j-1}/k$ times since we perform uniform exploration. Therefore $\left |L_{D^*}(h^{(j)}_i) - \EE [L_j(h^{(j)}_i)] \right | \leq \sqrt{k/2^{j-1}}$ and $L_{D^*}(h_j^*) \leq \theta^*  + 2\sqrt{k/2^{j-1}}$.

 Let $t_j$ be the last non-sampling round that $h^*_{j-1}$ is selected. This means that $t_j + 1$ is a sampling round and of phase $j+1$. The number of samples drawn up to phase $j$ and hence up to round $t_j$ is $\geq 2^{j+1}$. Due to this, $t_{j}+1$ being a sampling round implies that $\sum_{i < t_i} c_i \leq 2^{j+1}$. So we can bound the regret due to deploying $h^*_{j-1}$ in non-sampling rounds to be at most 
 \begin{align*}
     \sum_{i \leq t_j} c_i (L_{D^*}(h^*_{j-1}) - \theta^*) &\leq \frac {2^{j+1} \sqrt{k}} {2^{j}} \\
     &\in O(\sqrt{2^j})
 \end{align*}

Summing over regret due to non-sampling rounds over all phases $j$ from 1 to $\log (s_T^{2/3})$, total regret in non-sampling rounds is at most: $\sum_{j=1}^{\log (s_T^{2/3})} O(\sqrt{2^j}) \in O(s_T^{2/3})$.

Therefore the total regret over sampling and non-sampling rounds is $O(s_T^{2/3})$.

\end{proof}

\subsection{Proof of Lemma~\ref{prop:queriesuseless}}\label{proof:queriesuseless}
    \begin{proof}
For each $n, q$, we will construct a class of distributions such that for any learning algorithm with access to $n$ i.i.d. samples and $q$ queries, there is a distribution $D^*$ in the class for which the learning algorithm will have excess error at least $\theta^* + \Omega(1/\sqrt{n})$ where $\theta^*$ is the optimal error achieved by the class of one-dimensional half-spaces on $D^*$. This is the class $\mathcal{H}_{1d-HS}$ which we will refer by $\mathcal{H}$
\[\mathcal{H}_{1d-HS} = \{\mathds{1}\{x \geq \theta \} : \theta \in \mathbb{R}\} \cup \{\mathds{1}\{x \leq \theta \} : \theta \in \mathbb{R}\}. \]

As a construction for the lower bound, consider the class of distributions over a domain of $M$ points, each having a uniform marginal distribution supported on $m < M$ of those points. The labelling distribution $\Pr(y=1 | x)$ is $1/2 \pm 1/2\sqrt{n}$. We later describe how to choose $m, M$ so that the lower bound holds for $n, q$.

We will show that for any set of $q$ queries, there are two distributions $D_1, D_2$ such that
\begin{enumerate}
    \item All query values are the same for $D_1, D_2$.
    \item No algorithm can distinguish between $D_1, D_2$ with probability more than $1/2$ using $n$ samples drawn.
    \item Any classifier with error $ \min_{h \in \mathcal{H}_{1d-DS}} L_{D_1}(h) + O(1/\sqrt{n})$ on $D_1$ necessarily has error\\ $\min_{h \in \mathcal{H}_{1d-DS}} L_{D_2}(h) + \Omega(1/\sqrt{n})$ on $D_2$. 
\end{enumerate}
The above properties suffice to show the required lower bound. This is because with the above properties, a learning algorithm that achieves $o(1\sqrt{n})$ expected excess error necessarily distinguishes between $D_1$ and $D_2$ with probability at least $1/2$. Distinguishing between $D_1, D_2$ should not be possible with $n$ samples and $q$ queries if the above properties hold. 

Now let us show how to choose $m, M$ and construct $D_1, D_2$ to make the above properties hold. 

Each query assigns a value of 0 or 1 to each point. So there is a sequence of length $q$ indicating the labels the queries assign for each point. We can partition the domain into points having the same sequence of query labels. By constructing $D_1$, $D_2$, so that the number of points in each partition with labelling function $1/2 + 1/2\sqrt{n}$ is the same for $D_1$ and $D_2$, we can guarantee that all query values are the same for $D_1$ and $D_2$

Let us set $M$ so that $M > 2^q m$. Since there are $2^q$ query sequences, at least one of the partition sets has size $m$. Consider the $m$ points distributions $D_1, D_2$ are supported on to be the $n$ points of samples drawn and the remaining points are points from the partition of size $m$. There are at least $m - n$ points of the support from the partition of size $m$.

Let the points in the support that are in the query partition of size $m$ be $x_1 \leq \ldots \leq x_s$ where $s \geq m - n$. $D_1$ has probability of +1 label $1/2 + 1/\sqrt{n}$ for points $x_1, \ldots, x_{s/2}$ and probability of +1 label $1/2 - 1/\sqrt{n}$ for points $x_{s/2 + 1}, \ldots, x_{m}$. $D_2$ has probability of $+1$ label $1/2 - o1/\sqrt{n}$ for points $x_1, \ldots, x_{s/2}$ and probability of +1 label $1/2 + 1/\sqrt{n}$ for points $x_{s/2 + 1}, \ldots, x_{m}$.

Outside of this partition, let the labelling distribution of any other point in the support  be the same for $D_1, D_2$. By this construction, property (1) is satisfied. 

Restricted to the query partition, the errors on distributions $D_1, D_2$ sum up to $1/2 + 1/\sqrt{n}$. By choosing $m > 2n$, the query partition makes up at least half the fraction of the support. Therefore any classifier $h$ has $L_{D_1}(h) + L_{D_2}(h) = 1/4 + 1 / 2\sqrt{n}$.  and due to this property (3) holds. 

The KL divergence between the sampling distributions from $D_1, D_2$ is at most $\sqrt{n} \frac 1 {2 \sqrt{n}}$, and therefore with probability $\geq 1 / 2$, we cannot distinguish between $D_1$ and $D_2$ using $n$ samples.  This shows property (2) holds. 
\end{proof}

%%%%%%%%%%%%%%%%%%%%%%%%%%%%%%%%%%%%%%%%%%%%%%%%%%%%%%%%%%%%%%%%%%%%%%%%%%%%%%%
%%%%%%%%%%%%%%%%%%%%%%%%%%%%%%%%%%%%%%%%%%%%%%%%%%%%%%%%%%%%%%%%%%%%%%%%%%%%%%%

\section{Miscellaneous results and discussions}
\subsection{Treating error curves as upper bounds} \label{sec:errorBounds}
For most of our results we have made use of the structured form of error curves reflecting how expected error of a learned model is assumed to vary with the number of samples used for training. This structure is inspired by statistical minimax bounds and are upper bounds rather the true error curves. We designed contracts assuming the bounds to be actual error curves. Here we discuss what we can say about these contracts without assuming the bounds to be exact error curves.

From the principal's perspective, these contracts result in  accuracy that is just as good as that of learned models. However, the principal would end up paying the agent more than it could have if the principal knew the exact error curve. We can view the shape of the true error curve as another piece of information the principal is unaware of in addition to the optimal error. This hidden information results in more information rent but does not impact the accuracy of the model obtained from delegation.

The agent's perspective of what changes is more complicated. Our contracts assumed that the agent responded assuming that the upper bound was the true error curve. It may be reasonable that before starting the delegation process, the agent believes the upper bounds to be the true curves having no other frame of reference. However after starting to collect data, it is possible that the agent will learn more about the form of the true error curve and respond differently. This is similar to how the agent can learn the optimal error while executing the contract. Analyzing how this error curve learning occurs will allow us to design truly incentive-compatible contracts. However, learning the error curve shape is learning from a much broader class and is likely to be more challenging. 

\subsection{Variable label quality model}\label{sec:dataCollectionQuality}
The setting above models the scenario where the agent does not have the option to choose the quality of the data is collects. However, the agent might be able to control the quality of the data as a function of the cost per sample. We study a model of quality of data where the quality corresponds to the quality of the labels of the data. The quality parameter $q = 1 - 2 \eta$ captures the likelihood of the labels being correct. Here, $\eta \in (0, 1/2)$ is the probability of the label being incorrect. In this setting, the expected accuracy on the principal's test set from the agent collecting $n$ samples at quality level $q$ when the optimal error is $\theta$ is $1 - \theta - \frac 1 {q n^p}$ for some $p > 0$. We assume that the cost of collecting a single sample at quality level $q$ for the agent is given by $\alpha(q)$ a function increasing in $q$ and convex. So the cost for the agent of collecting $n$ samples at quality level $q$ is $C(n,q) = \alpha(q)n$. 

In this section, we provide results for $\alpha(q) = q^b + \alpha_0$ for $b > 0$. We can think of $\alpha$ as the cost of collecting an unlabelled sample and $q^b$ as the cost of labelling an unlabelled point. The main message of this section is that even though the quality of labels is an action that the agent chooses, effectively, this choice is not information that is private from the principal. It turns out that whatever the contract is, the utility-maximizing agent executes the contract by choosing a single quality value $q^*$. The principal can also compute $q^*$ so the quality is not a reflection of information asymmetry. Therefore, this regime is essentially the same as the one studied in the previous section.

\begin{theorem}[Constant quality level]\label{thm:qualityLevel}
    For any problem parameters $d, p, \alpha_0 > 0, b > 1$, when $\alpha(q) = q^b + \alpha_0$, for any expected accuracy $a$ the agent wishes to achieve, the agent chooses a constant $q^*$ that only depends on $\alpha_0, b$ as the quality level. 
\end{theorem}
\begin{proof}
If the agent aims to achieve an expected accuracy of at least $s$, then the agent chooses the number of samples and quality level by solving the following optimization problem:
\begin{equation*}
\begin{aligned}
\min_{q,n} \quad & \fParantheses{q^b + \alpha_0}  n\\
\textrm{s.t.} \quad & \theta + \frac {1} {qn^p} \leq 1 - a \\
  &0 \leq q \leq 1.    \\
\end{aligned}
\end{equation*}
The solution of this optimization problem can be calculated as follows:
\begin{align*}
    \mathcal{L} &= \fParantheses{q^b + \alpha}n + \lambda_1 \fParantheses{\theta - l + \frac 1 {qn^p}} + (\lambda_2 - \lambda_3)q  \\
    \nabla_n \mathcal{L} &= q^b + \alpha - \frac{p\lambda_1}{b n^{p+1}} \\
    \nabla_q \mathcal{L} &= b q^{b-1}n - \frac{\lambda_1}{q^2n^p} + \lambda_2 - \lambda_3.
    \intertext{If $\lambda_2^* = \lambda_3^* = 0$, we obtain}
    \Longrightarrow \lambda_1^* &= b {q^*}^{b + 1} n^{p+1} \\
    \Longrightarrow {q^*} &= \fParantheses{\frac \alpha {b-1}}^{\frac 1 b},
    \intertext{and $n^*$ is obtained by solving}
    \theta + \frac 1 {q^* {n^*}^{p}} &= l.
\end{align*}
If $(\alpha / (b-1))^{1/b}$ is not in $[0,1]$, then $\lambda_2^*$ or $\lambda_3^*$ is non-zero and $q^*$ is either 0 or 1.  
\end{proof}

\subsection{Closed-form solution for the two optimal error, hidden state problem}\label{app:closedFormOpAdvSel}
\newcommand{\thetalow}{\underline{\theta}}
\newcommand{\nhigh}{\overline{n}}
\newcommand{\nlow}{\underline{n}}
\newcommand{\thigh}{\overline{t}}
\newcommand{\tlow}{\underline{t}}
Here we solve the state-aware optimization problem~\ref{opt:advSelection} when there are two optimal errors, $\thetalow \leq \thetahigh$, with  a prior probability of $\nu$ for $\thetalow$. Let $L(n) = d/n^p$. We solve the following optimization problem and show that the solution is the optimal solution we are looking for. Note that this problem omits the incentive-compatibility constraint for the problem $\thetahigh$ and the participation constraint for the easy problem. 
\begin{equation*}
\begin{aligned}
\min_{\nlow,\nhigh,\tlow, \thigh} \quad & \nu(L(\nlow) + \beta \tlow) + (1-\nu)(L(\nhigh) + \beta \thigh)\\
\textrm{s.t.} \quad & \thigh - \frac {\alpha \nhigh} {(1 + \Delta {\nhigh^p})^{1/p}} - \tlow + \alpha \nlow \leq 0 \\
  &\alpha \nhigh - \thigh \leq 0.    \\
\end{aligned}
\end{equation*}

First note that this is a convex optimization problem, where the objective is convex by the convexity of the loss and the PC constraint is a linear constraint. All that is left is to check that the IC constraint is convex. This is the sum of linear terms and the term $\frac {- \alpha d^{1/p} \nhigh} {(d + \Delta {\nhigh^p})^2}$. The second derivative of this term is $\frac{3\alpha d^{1/p} \Delta }{2{\nhigh^p}(d + \Delta {\nhigh^p})^4}$. Since the second derivative is positive, the IC constraint is convex. 

Consider the Lagrangian
\[L(\nlow,\nhigh,\tlow, \thigh; \lambda_1, \lambda_2) = \nu(L(\nlow) + \beta \tlow) + (1-\nu)(L(\nhigh) + \beta \thigh) + \lambda_1 \fParantheses{\thigh - \frac {\alpha d^{1/p} \nhigh} {(d + \Delta {\nhigh^p})^{1/p}} - \tlow + \alpha \nlow} + \lambda_2 \fParantheses{\alpha \nhigh - \thigh}, \]
which has the following gradients:
\begin{align*}
    \nabla_{\nlow} L &= \nu L'(\nlow) + \alpha \lambda_1 &\text{(G1)} \\
    \nabla_{\tlow} L &= \nu \beta - \lambda_1 &\text{(G2)} \\
    \nabla_{\thigh} L &= (1 - \nu) \beta - \lambda_2 + \lambda_1 &\text{(G3)} \\
    \nabla_{\nhigh} L &= (1-\nu)L'(\nhigh) + \alpha \lambda_2 - \lambda_1 \fParantheses{\frac{\alpha d^{1/p}}{(d+\Delta {\nhigh^p})^{(p+1)/p}}} &\text{(G4)}
\end{align*}
To choose values $\nlow^*,\nhigh^*,\tlow^*, \thigh^*, \lambda_1^*, \lambda_2^*$ that satisfy the KKT conditions, first we set the gradients to zero: 
\begin{align*}
    L'(\nlow^*) &= -\alpha \beta &\text{(From (G1), (G2))} \\
    \lambda_1^* &= \nu \beta &\text{(From (G2)} \\
    \lambda_2^* &= \beta &\text{(From (G3) and value of $\lambda_1^*$)} \\
    (1-\nu) L'(\nhigh^*) + &\alpha\beta -\frac {\nu \alpha \beta}{(1+\Delta {\nhigh^p})^{(p+1)/p}} = 0  &\text{(From (G4) and values of $\lambda_1^*, \lambda_2^*$)} \\
    \Longrightarrow L'(\nhigh^*) &= - \frac {\alpha \beta} {1-\nu} \fParantheses{1 - \frac {\nu d^{1/p}}{(d+\Delta {\nhigh^{*p}})^{(p+1)/p}}}.
    \intertext{By complementary slackness,}
    \alpha \nhigh^* &= \thigh^* \\
    \tlow^* &= \alpha \fParantheses{\nlow^* - \frac {d^{1/p}\nhigh^*} {(d + {\nhigh^{*p}})^{1/p}} + \nhigh^*}.
\end{align*}

The contract described by $(\nlow^*,\nhigh^*,\tlow^*, \thigh^*)$ satisfies the properties of the second-best contract in the classical contract theory setting. We list these properties here:
\begin{enumerate}
    \item [P1] No output distortion for the easy problem: $\nlow^*$ is the solution of $L'(\nlow^*) = -\alpha \beta$ which is also the value of $n^{fb}$. So for the easy problem, the agent gathers the same number of samples as in the full information case.
    
    \item [P2] Downward distortion for the hard problem: 
    \begin{align*}
        L'(\nhigh^*) &= - \frac {\alpha \beta} {1-\nu} \fParantheses{1 - \frac {\nu d^{1/p}}{(d+\Delta {\nhigh^{*p}})^{(p+1)/p}}} \\
        &< -\alpha \beta \\
        &= L'(n^{fb}).
    \end{align*}
    So $\nhigh^* < n^{fb}$. For the hard problem, the agent gathers fewer samples than in the full information case.
    
    \item [P3] When the problem is easy, the agent gets positive information rent:
    \begin{align*}
        \tlow^* - \alpha \nlow^* &= \nhigh^* - \frac {d^{1/p}\nhigh^*} {(d + \Delta {\nhigh^{*p}})^{(p+1)/p}} \\
        &> 0.
    \end{align*}
\end{enumerate}

We now check that the contract that is the solution to the above optimization problem also satisfies the omitted constraints. First we start with the participation constraint for the easy problem. By the positive information rent property (P3) we know that $\alpha \nlow^* < \tlow^* $. Next consider the incentive-compatibility constraint for the hard problem. We only need to check when $\Delta {\nlow^{*p}} < d$. Otherwise, the IC constraint automatically holds. The difference in agent's utility between choosing the $\overline{\ell}^* = \thetahigh + d/{\nhigh^*}^p$ option and the $\underline{\ell}^* = \thetalow + d/{\nlow^*}^p$ option is:
\begin{align*}
    &-\alpha \nlow^* + \thigh + \frac{\alpha d^{1/p}\nlow^*}{\fParantheses{d - \Delta {\nlow^{*p}}}^p} - \tlow^* \\
    &=  -\frac{\alpha d^{1/p} \nhigh^*}{(d + \Delta {\nhigh^{*p}})^{1/p}} + \alpha \nhigh^* + \frac{\alpha d^{1/p} \nlow^*}{\fParantheses{d - \Delta {\nlow^{*p}}}^{1/p}} - \alpha \nlow^* &\text{(Using values of $\tlow^*, \thigh^*$)} \\
    &= \alpha \fParantheses{\frac {d^{1/p} \nlow^*} {(d - \Delta {\nlow^{*p}})^{1/p}} - \nlow^* + \frac {d^{1/p} \nhigh^*} {(d + \Delta {\nhigh^{*p}})^{1/p}} - \nhigh^*} \\
    &\geq \alpha \nhigh^* \fParantheses{\frac {1} {(1 - \Delta {\nlow^{*p}})^{1/p}}  + \frac {1} {(1 + \Delta {\nlow^{*p}})^{1/p}} - 2} &\text{(Since $\nlow^* < \nhigh^*$)}
    \intertext{Note that the function $\frac {d^{1/p}} {(d-x)^q} + \frac {d^{1/p}} {(d+x)^q}$ is increasing in the interval $[0,1)$ for every $q$. The derivative of that function is $q \fParantheses{\frac {d^{1/p}} {(d-x)^{q+1}} - \frac {d^{1/p}} {(d+x)^{q+1}}}$. This is nonnegative and lies in $[0,1)$.}
    &\geq 0 &\text{(Since we assume $0 < \Delta {\nlow^{p*}} < 1$)}.
\end{align*}
This solution finds the optimal contract under hidden state. 

\subsubsection{Separating contracts}\label{app:sep}
To be able to compute any separating contract, it suffices to solve the above optimization problem with the additional constraint $\nlow, \nhigh \geq n_0$ for some $n_0 \geq 0$. The new optimizers $\nlow(n_0),\nhigh(n_0),\tlow(n_0).\thigh(n_0)$ are as follows:
\begin{align*}
    \nlow(n_0) &= \max(\nlow^*, n_0) \\
    \nhigh(n_0) &= \max(\nhigh^*, n_0) \\
    \tlow(n_0) &= \alpha \nlow(n_0) \\
    \thigh(n_0) &= \alpha \fParantheses{\nlow(n_0) - \frac {d^{1/p}\nhigh(n_0)} {(d + {\nhigh(n_0)})^{1/p}} + \nhigh(n_0)}.
\end{align*}

\subsubsection{Optimal pooling contract}\label{app:pool}
The optimal pooling contract optimizes over $\nhigh$. $t = \alpha \nhigh$. $\nlow$ is chosen such that
\begin{align*}
    \thetalow + \frac d {\nlow^p} &= \thetahigh + \frac{d}{\nhigh^p} \\
    \Longrightarrow \nlow = \frac{d^{1/p} \nhigh}{(d + \Delta \nhigh^p)^{1/p}}.
\end{align*}
Thus the optimization problem is choosing $\nhigh$ to be the minima of 
\[\nu \frac{d}{\fParantheses{\frac{d^{1/p} n}{(d + \Delta n^p)^{1/p}}}^p} + (1-\nu)\frac{d}{n^p} + \alpha \beta n.\]

\subsection{Tightness of linear contracts approximation}
Our main result (Theorem~\ref{cor:advSelLinear}) gave a linear contract that provably approximates the optimal contract up to a constant factor. This approximation factor stated in Proposition~\ref{thm:LinContractApprox} is also tight as stated in the following theorem, which shows that there is a problem instance for which no linear contract can do better than a given approximation factor. The problem instance for which the approximation ratio is tight is one that has deterministic test error distribution, which arises when the size of the test set tends to infinity.

\begin{theorem}[Tightness of approximation bound] \label{thm:tightApprx}
For every $\theta \in [0,1), p, d > 0$, there are problem parameters $\alpha, \beta > 0$ such that for the problem instance with these parameters, all linear contracts have at most  $1 - \frac{1}{(p+1)^\frac{p+1}{p}}$ times the optimal utility.
\end{theorem}
\begin{proof}

We will show that there exist $\alpha, \beta$ such that the contract $\frac 1 {\beta(p+1)^{\frac{p+1}{p}}}$ is the optimal contract chosen by the principal. This contract will also satisfy the participation constraint for our chosen values of $\alpha, \beta$. Recall that the participation constraint is 
\begin{align*}
    \frac 1 {\beta(p+1)^{\frac{p+1}{p}}} &\geq \frac{\alpha d^{\frac 1 p}}{p} \cdot \fParantheses{\frac{p+1}{1-\theta}}^{\frac{p+1}{p}} \\
    \equiv 1-\theta &\geq \fParantheses{\frac{\alpha \beta d^{\frac 1 p}}{p}}^{\frac{p}{p+1}}(p+1)^2.
\end{align*}

The principal chooses the contract that sets the derivative of the above quantity to zero as long as that contract satisfies the participation constraint. If setting $\frac 1 {\beta(p+1)^{\frac{p+1}{p}}}$ yields a zero derivative and it satisfies the participation constraint, then it is the optimal linear contract. The derivative relative to $c$ is
\[(1-\beta c) \fParantheses{\frac{\alpha d^{\frac 1 p}}{p}}^\frac{p}{p+1} \cdot \frac{p}{p+1} \cdot \frac{1}{c^{\frac{2p+1}{p+1}}} - \beta \fParantheses{1 - \theta - \fParantheses{\frac{\alpha d^{\frac 1 p}}{p}}^\frac{p}{p+1}}.\]
Setting $c = \frac 1 {\beta(p+1)^{\frac{p+1}{p}}}$, the derivative is
\begin{align*}
    &\fParantheses{1 - \frac{1}{(p+1)^\frac{p+1}{p}}} \frac{p}{p+1}\fParantheses{\frac{\alpha d^{\frac 1 p}}{p}}^\frac{p}{p+1} \beta^\frac{2p+1}{p+1}(p+1)^\frac{2p+1}{p} \\
    &- \beta \fParantheses{1 - \theta - (p+1)\fParantheses{\frac{\alpha \beta d^{\frac 1 p}}{p}}^\frac{p}{p+1}}.
\end{align*}
We can choose $\alpha, \beta$ to set this derivative to zero by choosing $\alpha, \beta$ satisfying:
\begin{align*}
    1-\theta &= \fParantheses{\frac{\alpha \beta d^{\frac 1 p}}{p}}^\frac{p}{p+1} (p+1) \fParantheses{1 + \fParantheses{1 - \frac{1}{(p+1)^\frac{p+1}{p}}}p(p+1)^\frac{1}{p}}.
    \intertext{Note that for every $p > 0$, $\fParantheses{1 + \fParantheses{1 - \frac{1}{(p+1)^\frac{p+1}{p}}}p(p+1)^\frac{1}{p}} > p+1$. So,}
    1-\theta &\geq \fParantheses{\frac{\alpha \beta d^{\frac 1 p}}{p}}^\frac{p}{p+1} (p+1)^2.
\end{align*}
This shows that there are problem parameters that make $c^* = \frac 1 {\beta(p+1)^{\frac{p+1}{p}}}$ the optimal linear contract. In the proof of Proposition~\ref{thm:LinContractApprox}, we showed that this linear contract achieves at least $1- \frac{1}{(p+1)^{\frac{p+1}{p}}}$ times the optimum utility. When the problem involves a deterministic mapping between the number of samples and the observed accuracy, this ratio is exact. 
\end{proof}

\end{document}